\documentclass[11pt, a4paper]{article}
\usepackage{lineno,hyperref}

\usepackage[hmarginratio=1:1,top=25mm,right=18mm,left=18mm,bottom=25mm,columnsep=20pt, headsep=7mm]{geometry}

\usepackage[colorinlistoftodos]{todonotes}

\newcommand{\sMean}[1]{\hat{\bbE}{#1}}										

\providecommand*{\Dom}[1]{{\rm dom}}							
\providecommand{\argmin}{\operatorname*{argmin}}				

\newcommand{\curlyb}[1]{\left\{{#1}\right\}}                    

\providecommand{\CB}{{\cal B}}
\providecommand{\CC}{{\cal C}}

\providecommand{\CJ}{{\cal J}}

\providecommand{\CN}{{\cal N}}
\providecommand{\CO}{{\cal O}}

\providecommand{\CV}{{\cal V}}

\providecommand{\bbE}{\mathbb{E}}

\providecommand{\bbR}{\mathbb{R}}

\usepackage{subfig}
\usepackage{xcolor}
\usepackage{bbm}
\definecolor{ForestGreen}{RGB}{34,139,34}

\newcommand{\newmcommand}[2]{\newcommand{#1}{{\ifmmode {#2}\else\mbox{${#2}$}\fi}}}
\newcommand{\newmcommandi}[2]{\newcommand{#1}[1]{{\ifmmode {#2}\else\mbox{${#2}$}\fi}}}
\newcommand{\newmcommandii}[2]{\newcommand{#1}[2]{{\ifmmode {#2}\else\mbox{${#2}$}\fi}}}
\newcommand{\newmcommandiii}[2]{\newcommand{#1}[3]{{\ifmmode {#2}\else\mbox{${#2}$}\fi}}}

\newmcommandi{\paren}{\left({#1}\right)}

\newcommand{\x}{x}
\renewcommand{\v}{{v}}

\newcommand{\reals}{\bbR}

\newcommand{\W}{{W}}

\newcommand{\Laplacian}{{L}}

\newcommand{\nk}{{\hat k}}

\newcommand{\M}{M}
\newcommand{\X}{{X}}

\newcommand{\ind}{\mathbbm{1}}
\newcommand{\R}{\mathbb{R}}

\newcommand{\calC}{\mathcal{C}}

\usepackage{amsfonts, amsmath, amssymb, amsthm}
\usepackage{mathtools}
\usepackage{enumerate,enumitem}
\usepackage{upgreek}
\usepackage{bbm}

\usepackage{xcolor}
\usepackage{graphics}

\usepackage{array}
\usepackage{etoolbox}

\usepackage{caption,float}
\usepackage{bm}
\usepackage[section]{placeins} 

\usepackage{multirow}

\usepackage{siunitx}
\usepackage{textcomp} 
\usepackage{gensymb}

\usepackage{algorithm}
\usepackage{algorithmic}

\theoremstyle{plain}
	\newtheorem{theorem}{Theorem}[section]
	\newtheorem{proposition}[theorem]{Proposition}
	\newtheorem{lemma}[theorem]{Lemma}

	\newtheorem{remark}[theorem]{Remark}
	\newtheorem{definition}[theorem]{Definition}
	
	\newtheorem{problem}{Problem}

\allowdisplaybreaks

\usepackage{layout}
\numberwithin{equation}{section}

\begin{document}

\title{\vspace{0.8cm}Effective resistance in metric spaces\vspace{0.5cm}}

\newcommand{\footremember}[2]{%
    \footnote{#2}
    \newcounter{#1}
    \setcounter{#1}{\value{footnote}}%
}
\newcommand{\footrecall}[1]{%
    \footnotemark[\value{#1}]%
} 

\author{%
    Robi Bhattacharjee\footremember{ucsdRobi}{Department of Informatics, University of California San Diego, 9500 Gilman Dr, La Jolla, CA 92093, United States}%
    \and
  Alexander Cloninger\footremember{ucsdAlex}{Department of Mathematics, University of California San Diego, 9500 Gilman Dr, La Jolla, CA 92093, United States}%
  \and 
  Yoav Freund\footremember{ucsdYoav}{Department of Computer Science and Engineering, University of California San Diego, 9500 Gilman Dr, La Jolla, CA 92093, United States}%
  \and
 Andreas Oslandsbotn\footremember{uio}{Department of Informatics, University of Oslo, Problemveien 7, 0315 Oslo, Norway}\footremember{simula}{Simula Research Laboratory, Kristian Augusts gate 23, 0164 Oslo, Norway}%
  }
\date{}

\maketitle
\begin{abstract}
Effective resistance (ER) is an attractive way to interrogate the structure of graphs. It is an alternative to computing the
eigenvectors of the graph Laplacian.

One attractive application of ER is to point clouds, i.e.
graphs whose vertices correspond to IID samples from a distribution over a metric space. Unfortunately, it was shown that 
the ER between any two points converges to a trivial quantity that holds no information about the graph's structure as the size of the sample increases to infinity.

In this study, we show that this trivial solution can be circumvented by
considering a region-based ER between pairs of {\em small regions} rather than pairs of {\em points}
and by {\em scaling the edge weights} appropriately with respect to the
underlying density in each region. By keeping the regions fixed, we
show analytically that the region-based ER converges to a non-trivial limit as the
number of points increases to infinity. Namely the ER on a metric space. We support our theoretical findings with numerical experiments.
\end{abstract}

\section{Introduction}

A fundamental task of data science is to model the structure of point clouds embedded in a high-dimensional  ambient space. A common approach to this task is to represent the data as a graph where data points are considered vertices, and neighborhood information on the point cloud is encoded in edges connecting the vertices. The edges are often weighted based on the relative distance between points and are usually restricted to local neighborhoods. For simplicity of the introduction, we assume that an edge connects any two vertices whose distance is at most $\gamma>0$.~\footnote{More general definitions of edge weights will be given later.} 

Measuring the lengths of paths on such graphs is a key component of many methods that seek to characterize point clouds. In the following, we compare two of the most popular graph metrics used for this purpose:

\begin{itemize}
\item {\bf Shortest path} defines the distance between two vertices
  as the minimal number of edges that must be traversed to get from
  one vertex to the other. This is the most straightforward and
  intuitive measure of distance and corresponds to the geodesic
  distance when the point cloud lies on a differentiable
  manifold. However, the shortest path is sensitive to noise and the
  subtraction or addition of single points. It is, therefore, unreliable
  in the context of random point clouds.

\item{\bf Effective resistance (ER)}, also called {\em commute time} is
  significantly more stable than the shortest path distance as it 
  considers all possible paths between two points instead of only the shortest.
  Specifically, the {\em hitting time} of vertex $B$ from vertex $A$
  is the expected number of edges traversed by a random walk starting
  at $A$ and ending at $B$. The {\em commute time} is the sum of the
  hitting time from A to B and from B back to A. The {\em effective resistance} is a metric defined using analogy to resistance networks and is equal to the commute time up to a constant.
\end{itemize}

Figure \ref{fig:halfmoon_ER_vs_shortestpath} shows the difference between the shortest path and the ER distance on a point cloud shaped as a high-density half-moon over a noisy background. The figure illustrates how the shortest path distance would follow the noisy background, while the ER would follow the half-moon arch.

\begin{figure}
    \centering
    \includegraphics[width=0.5\textwidth]{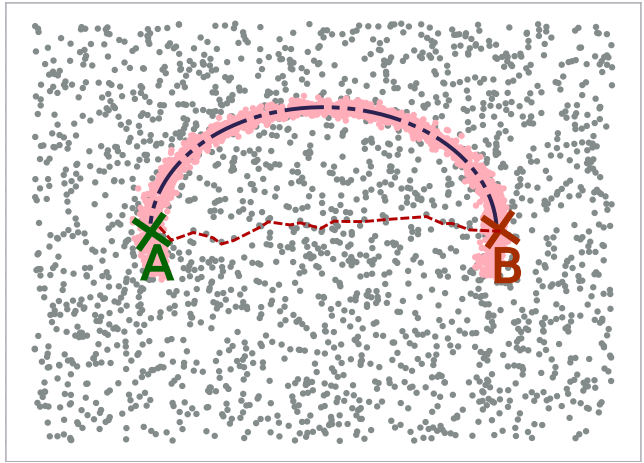}
    \caption{Comparison between ER distance and shortest path distance on a graph constructed on a point cloud. The point cloud consists of a high-density half-moon over a noisy background. The ER distance is the blue dotted line following the half-moon arch. The shortest path distance is the red dotted line across the center of the half-moon.}
    \label{fig:halfmoon_ER_vs_shortestpath}
\end{figure}

ER has been used in a number of applications in machine learning on
topics such as graph sparsification \cite{spielman2011graph}, online
learning \cite{herbster2006prediction}, detecting community structure
\cite{zhang2019detecting} and dimensionality reduction
\cite{ham2004kernel}. It has also been used in several other fields
ranging from bioinformatics \cite{forcey2020phylogenetic} to
electronics \cite{tauch2015measuring, kocc2014impact,
  wang2015network, cavraro2018graph}


A long-standing issue with using ER for point clouds is that the ER
distance between points converges to a trivial limit as the number of
sampled points increases. This problem has been described in a long
line of works \cite{lovasz1993random, boyd2005mixing, avin2007cover,von2010getting, von2014hitting}. The problem is that in the large graph limit, the ER between two points A and B depends only on the degree (number of neighbors) of $A$ and $B$. This means that ER distances are not useful for the analysis of the shape of point clouds.

The trivial limit can be described as follows:
\begin{problem}[Von-Luxburg limit] \label{problem:trivial-limit}
  The ER between two nodes $i,j$ in a graph converges as the number of
  nodes in the graph increases to $1/d_i + 1/d_j$, where $d_i, d_j$
  are the degrees of respectively node $i$ and $j$.
\end{problem}

While the statement of the problem is asymptotic,
Von-Luxburg et al.~\cite{von2014hitting} shows that the asymptotic behavior begins already for
moderately sized graphs, with the number of nodes in the order of
$1000$. Moreover, for increasing dimension of the embedding space, this convergence is even faster.

This study aims to develop a new formulation of ER that does not
suffer from Problem~\ref{problem:trivial-limit} and can thereby be
used to characterize large point cclouds.

The key idea is to consider the RE between small regions, rather than individual points. The number of points in a region scales linearly with the total number of points, which avoids the collapse of the RE in the limit. While this is in principle a simple transformation, some care is required to prove that it works as desired.

Our contributions can be summarized as follows:
\begin{itemize}
\item We alleviate the problem described by Von-Luxburg et al.~\cite{von2010getting} by introducing the concept of a region-based ER between a source and sink region, combined with an appropriate scaling as the sample size grows. 
\item We prove the existence and convergence of region-based ER to a non-trivial limit.
\item We prove that region-based ER is a distance metric. In particular, region-based ER satisfies the triangle inequality.
\item We support our theoretical findings with several numerical experiments. 
\end{itemize}

The remainder of this paper is organized as follows. In Section \ref{section:resistor_graphs}, we introduce the concept of a resistor graph, the standard definition of ER, and the definition of ER between sets. In Section \ref{section:resistor_graphs_in_metric_spaces}, we introduce the concept of a metric graph and extend the concept of effective resistance to metric spaces. In Section \ref{section:convergence_of_ER}, we show how a finite sample region-based ER converges, with the appropriate scaling, to the limit object region-based ER on a metric space. We provide further results for the region-based ER in Section \ref{section:properties_of_regionbasedER}, where we show that it corresponds to the ER between sets and is a distance metric. Meanwhile, Section \ref{section:computational_complexity_considerations} describes a strategy to control the computational complexity of the ER calculation, using an $\varepsilon$-cover combined with a suitable scaling of the graph weights. Finally, section \ref{section:experiments} provide several numerical experiments supporting our theoretical findings. 

\subsection{Related work}
In our work, we show that our region-based ER can be thought of as ER between sets. An extension of the ER to ER between disjoint sets was introduced in Song et al.~\cite{song2019extension} for application on signed graphs. Song et al.~\cite{song2019extension} showed that ER between sets is a convex function of the edge weights. Furthermore, it was established that effective resistance between disjoint sets is monotonically increasing w.r.t. decreasing set size. Our contribution is to show that ER between sets can be generalized to ER over a metric space. We prove convergence to a non-trivial limit as the graph size increases and give numerical evidence that this limit is meaningful. Furthermore, we prove the triangle inequality and show that region-based ER is a distance metric.

\section{Resistor Graphs}\label{section:resistor_graphs}
In this section, we review several key ideas and concepts about \textit{resistor graphs} and the effective resistance. Finally, we introduce the notion of effective resistance between sets.

We start by formulating our setting. Let $G = (\X,\W)$ be an un-directed, weighted graph with nodes $\X = \{x_1,\ldots, x_n\}$, and edge weights $\W_{i,j}$ and let $L=D-W$ be the graph Laplacian, where $D$ is the degree matrix $D_{ii} = \sum_j W_{ij}$. 

The graph can be thought of as an electrical network where each edge $(x_i, x_j)$ has a non-negative \textit{resistance} $R(x_i,x_j) = 1/W_{ij}$. With further analogy to electrical circuits, we can interpret a function $v: X \rightarrow \bbR$ on the graph's vertices as a voltage $v(x_i)$ and assign to each edge a signed {\em current} $J_{i,j}=-J_{j, i}$, which can be related to the resistance and voltages through Ohm's law. The relation can be written as 
\begin{equation}
      \v(x_i)-\v(x_j) = R(x_i,x_j) J_{i,j} \quad \text{or alternatively} \quad  J_{i,j} = W_{ij}(\v(x_i)-\v(x_j)).
\end{equation}
Meanwhile, from Kirchhoff's law, the sum of currents entering a node $i$ must be zero, namely
\begin{equation}
\sum_{j\sim i} J_{i,j} = J_{ext,i},
\end{equation}
where $J_{ext, i}$ is an external current that can be either a source, a sink, or zero if the node is unconstrained (no external source applied). Combining these laws, we have that
\begin{equation}
    (Lv)_i =\sum_{j\sim i} W_{ij}(v(x_i) - v(x_j)) = J_{ext,i}
    \label{eq:combined_kirchhoff_and_ohm}
\end{equation}

Furthermore, since energy transfer in an electrical circuit is voltage times charge. We can define the {\em energy} of the voltage $\v$ as
 \begin{equation} \label{eqn:energy}
   E(\v) \doteq  \sum_{x_i, x_j \in X} \W_{i,j} (\v(x_i)-\v(x_j))^2 =
   \v^T \Laplacian \v
 \end{equation}

\subsection{Effective resistance}\label{section:ER}
The ER is a measure for calculating distances on graphs \cite{klein1993resistance}, which considers all possible paths between two nodes, as opposed to the shortest-path metric. As such, the ER captures the graph's structure more carefully, which can be advantageous in many applications. Using the analogy to electrical circuits, effective resistance can be defined in two equivalent ways based on $R(x_i, x_j) = \Delta V / J$, where $\Delta V$ is the voltage difference and $J$ the total current flowing between the two nodes. In the \textbf{voltage difference formulation}, the current flow is constrained to unity, and the resistance is given in terms of the voltage difference between the nodes. In the \textbf{current flow formulation} the voltage difference is constrained to unity, and the resistance is given in terms of the inverse of the total current flow between the nodes. The following definitions formalize these approaches.

\begin{definition}[Effective resistance (voltage difference formulation)]\label{def:ER_voltage_diff_def}
The effective resistance $R(x_i, x_j)$ between two nodes $x_i, x_j\in X$ in a graph is the voltage difference between the nodes when a current of one ampere is injected between the source node $x_i$ and extracted from the sink node $x_j$. 
\end{definition}

\begin{definition}[Effective resistance (current flow formulation)]\label{def:ER_current_flow_def}
The effective resistance $R(x_i, x_j)$ between two nodes $x_i, x_j\in X$ in a graph is the inverse of the current flow between the nodes with the boundary conditions $v(x_i)=1$ and $v(x_j)=0$. 
\end{definition}

From Definitions \ref{def:ER_voltage_diff_def} and \ref{def:ER_current_flow_def}, we know that $R(x_i, x_j)$ can be found from the voltage or current in the system when the other is appropriately constrained. However, we still need a way to explicitly calculate these quantities. Several approaches exist for calculating the ER, and a summary of different formulations can be found in Theorem 4.2 in Jørgensen and Erin \cite{jorgensen2008operator}. However, in this work we restrict ourselves to the two formulations that follow most naturally from Definitions \ref{def:ER_voltage_diff_def} and \ref{def:ER_current_flow_def} respectively.

\begin{proposition}[Voltage difference formulation]\label{prop:poisson_formulation_ER} The effective resistance between nodes $x_i, x_j$ corresponds to $R(x_i, x_j) = v(x_i) - v(x_j) = E_{min}$, where $v$ is the function that minimizes the energy
\begin{align*}
    \begin{split}
       \min_{v} & \quad \sum_{x_i, x_j \in X} \W_{i,j} (\v(x_i)-\v(x_j))^2 \\
        \text{Subject to} & \quad (L v)_i = 1, \quad  (L v)_j = -1, \quad (L v)_i = 0, \forall i\in X\backslash \{x_i,x_j\}
    \end{split}
\end{align*}
\end{proposition}
\begin{proof}
See Theorem 4.2, Jørgensen and Erin \cite{jorgensen2008operator}.
\end{proof}

\begin{proposition}[Current flow formulation]\label{prop:dirichlet_formulation_ER} The effective resistance between nodes $x_i, x_j$ corresponds to $R(x_i, x_j) = 1/J_{tot}$, where 
\begin{equation*}
    J_{tot} = \sum_{j\in X} W_{ij}(v(i) - v(j))
\end{equation*}
and $v$ is the function that minimizes the Dirichlet energy
\begin{align*}
    \begin{split}
       \min_{v} & \quad \sum_{x_i, x_j \in X} \W_{i,j} (\v(x_i)-\v(x_j))^2 \\
        \text{Subject to} & \quad v(x_i) = 1, \quad  v(x_j) = 0
    \end{split}
\end{align*}
\end{proposition}
\begin{proof}
See Theorem 4.2, in Jørgensen and Erin \cite{jorgensen2008operator}.
\end{proof}

We note that the standard formulation of effective resistance
\begin{equation*}\label{eq:compact_ER}
    R(x_i, x_j) = (e_i-e_j)^\top L^\dagger(e_i-e_j)
\end{equation*}
can easily be derived from $R(x_i, x_j) = v(x_i) - v(x_j) = (e_i-e_j)^\top v$ and the constraints $Lv = e_i-e_j$ in Proposition \ref{prop:poisson_formulation_ER}, which gives $v = L^\dagger(e_i-e_j)$. This relation shows how ER is an euclidean distance matrix \cite{ghosh2008minimizing}. Furthermore, it show how the distance between nodes using ER, can be computed without explicitly calculating the voltage functions $v$. Instead, it suffices to solve for $L^{\dagger}$, the pseudo-inverse of $L$. Methods have also been proposed for distributed computation \cite{gillani2021queueing}. 

\begin{lemma}Effective resistance satisfies the triangle inequality and is a distance metric on graphs
\end{lemma}
\begin{proof}
    See e.g. Jørgensen and Erin \cite{jorgensen2008operator}, Ghosh et al.~\cite{ghosh2008minimizing} or Klein and Randi{\'c}\cite{klein1993resistance}.
\end{proof}

\begin{remark}
We note that the \textbf{voltage difference formulation} corresponds to a Poisson problem since the constraints on the current are applied as an inhomogenous term on the Laplace equation $Lv = e_1 - e_2$. Similarly, the \textbf{current flow formulation} corresponds to a Dirichlet problem since it solves a homogenous Laplace equation with the constraints applied as boundary conditions.
\end{remark}



\subsection{Effective Resistance Between Sets}\label{section:ER_sets}
In this study, we propose considering the ER between
small regions rather than pairs of points to achieve a non-trivial limit in the metric graph setting. Our first step in this regard is to extend the concept of effective resistance to effective resistance between sets. Whereby sets, we mean subsets of the graph nodes $X$. The effective resistance between sets was defined in Definition 2 \cite{song2019extension}, in terms of the Schur complement of the Laplacian. For the analysis in this paper, we consider the current flow interpretation from Song et al.~\cite{song2019extension} and use this to write down an equivalent formulation that generalizes the current flow formulation from Proposition \ref{prop:dirichlet_formulation_ER} to the effective resistance between sets. We consider the sets $X_a, X_b \subset X$ and $X_c = X\backslash (X_a \cup X_b \cup X_z)$ and let $R^s(X_a, X_b)$, defined in Definition \ref{def:ER_sets_current_flow_def}, denote the effective resistance between two sets. Proposition \ref{prop:dirichlet_formulation_ER_sets} tells us how we can explicitly calculate $R^s(X_a, X_b)$.

\begin{definition}[Effective resistance between sets (current flow formulation)]\label{def:ER_sets_current_flow_def}
The effective resistance $R^s(X_a, X_b)$ between two non-empty disjoint sets $X_a, X_b\subset X$ in a graph is the inverse of the current flow between the two subsets with the boundary conditions $v(x_i)=1,\, \forall i \in X_a$ and $v(x_i)=0,\,\forall i \in X_b$.  
\end{definition}

\begin{proposition}[Current flow formulation on sets]\label{prop:dirichlet_formulation_ER_sets}
The effective resistance between the non empty disjoint subsets $X_s, X_g \subset X$ corresponds to $R^s(X_s, X_g) = 1/J_{tot}$, where 
\begin{equation*}
    J_{tot} = \sum_{i\in X_s}\sum_{j\in X} W_{ij}(v(i) - v(j))
\end{equation*}
and $v$ is the function that minimizes the energy
\begin{align*}
    \begin{split}
       \min_{v} & \quad \sum_{x_i, x_j \in X} \W_{i,j} (\v(x_i)-\v(x_j))^2 \\
        \text{Subject to} & \quad v(x_i) = 1\, \forall i\in X_s, \quad  v(x_i) = 0,\, \forall i \in X_g
    \end{split}
\end{align*}
\end{proposition}
\begin{proof}
Definition 2 and Theorem 1 in Song et al.~\cite{song2019extension}.
\end{proof}



\section{Effective Resistance over metric spaces}\label{section:resistor_graphs_in_metric_spaces}
Our goal is to extend the concept of effective resistance to metric spaces. To this end, we start by defining a particular type of graph, namely graphs constructed from samples drawn from a distribution over a metric space. Let $(\M, d)$ be a compact metric space and $\mu$, a probability measure over $\M$.

Let $k: \M \times \M \rightarrow [0, 1]$ be a kernel function that defines what it means for two points to be "near" each other. 
Commonly used kernel functions are:
\begin{itemize}
\vspace{0.25cm}
    \item The radial kernel: $k_r(x, y) = \mathbbm{1}(d(x, y) \leq r)$ where $r > 0$ is some fixed radius.
    \vspace{0.25cm}
    \item The Gaussian kernel: $k_\sigma(x,y) = \exp(\frac{-d(x,y)^2}{2\sigma^2}),$  where $\sigma > 0$ is the fixed temperature parameter.
\vspace{0.25cm}
\end{itemize}


Next, let $M_s \subseteq M$ and $M_g \subseteq M$ be two disjoint measurable subsets of $M$.  As before, we let $k: M \times M \to [0, 1]$ be a kernel function. 

To define the effective resistance between $M_s$ and $M_g$, we begin by defining the energy-minimizing voltage induced by $\mu$ and $k$ over these sets.

\begin{definition}\label{defn:energy_minimizing_voltage}
Let $M_s, M_g \subset M$ be measurable disjoint subsets, $k$ a kernel function, and $\mu$ a probability measure over $M$. Let $V_{M_s, M_g}$ be the set of all measurable functions $v: M \to [0, 1]$ with $v(x) = 1$ for all $x \in M_s$ and $v(x) = 0$ for all $x \in M_g$. For any such $v$, define its energy as $$E(v) = \int_{M} \int_{M}  k(x, y) (\v(x)-\v(y))^2d\mu(x) d\mu(y).$$ Then we say that $v^* \in V_{M_s, M_g}$ is an energy minimizing voltage if it is a  global minimum of $E(v)$, meaning that $$v^* = \argmin_{v \in V_{M_s, M_g}} E(v).$$ 
\end{definition}

Observe that Definition \ref{defn:energy_minimizing_voltage} does not necessarily imply a unique energy minimizing voltage induced by $M_s$ and $M_g$. However, under certain technical conditions on $k, \mu$, and $M$, we can show that uniqueness holds, which will be crucial for defining the effective resistance between $M_s$ and $M_g$.

\subsection{Technical conditions of  $M$, $\mu$, and $k$}

We begin with the notion of a normalized kernel, which integrates to $1$ over any fixed point $x \in M$. 

\begin{definition}\label{defn:norm_kernel}
The normalized kernel of $k$, denoted $\nk: M \times M \to R$, is defined as $$\nk(x, y) = \frac{k(x,y)}{\int_\M k(x, z)d\mu(z)}.$$
\end{definition}

A normalized kernel can be thought of as the natural analog of a degree normalized weight matrix. We can easily verify that $\int \nk(x,y)d\mu (y) = 1$ for all $x$. 

Next, we generalize the notion of \textit{adjacency} to the metric setting.

\begin{definition}
Let $A \subseteq \M$ a measurable set and $\alpha > 0$. Define $C_\alpha(A)$ as the set of all $x$ such that $\int_A \nk(x,y)d\mu(y) > \alpha$. 
\end{definition}

$C_\alpha(A)$ can be thought of as a continuous generalization of the set of ``neighbors" of $A$. In the finite setting $C(A)$ would comprise of vertices that are adjcaent to some vertex in $A$.  The parameter $\alpha$ provides a lower bound on the degree of connectivity.

Next, we continue this to develop an analogous generalization of vertices that are path connected to $A$. To do so, we first define the convolution operation over kernel similarity functions.

\begin{definition}\label{defn:convolution}
Let $p, q$ be kernel functions $p: \M \times \M \to [0, 1]$ , $q: \M \times \M \to [0, 1]$. Then their convolution is the function $p \circ q: \M \times \M \to [0, 1]$ defined by $$(p \circ q)(x, y)  = \int_M p(x, z)q(z, y)d\mu(z).$$ We let $p^{(i)}$ denote $p \circ p \circ \dots p$ repeated $i$ times.
\end{definition}

The key idea is that the kernel $\nk^{(i)}$ corresponds to paths of length $i$. As a result, we now define $C_\alpha^i$ accordingly. 

\begin{definition}\label{defn:path}
Let $A \subseteq \M$ a measurable set. Then $C_\alpha^i(A)$ denotes the set of all $x$ for which $\int_A \nk^{(i)}(x, y)d\mu(y) \geq \alpha.$
\end{definition}

We will now restrict our interest to metric spaces $\M$ for which $\M = C_\alpha^m(A)$ for some fixed integer $m > 0$, and real number $\alpha > 0$. This condition essentially means that the graphs we consider are both path connected and have bounded diameter, and  the parameter $\alpha$ implies that this connection has a degree of robustness. We note that these assumptions are extremely mild: for example, any compact manifold using a radial or Gaussian kernel can be easily shown to satisfy them.

\subsection{Defining the effective resistance}

We now show that under the previously discussed technical conditions, there exists a unique energy minimizing voltage with respect to $M_s, M_g$, $k$, and $\mu$. 


\begin{proposition}\label{prop:unique_voltage_solution}
Let $k$ be a kernel, $M$, a metric space, and $\mu$ a measure over $M$. Let $M_s$ and $M_g$ be measurable disjoint subsets of $M$. Suppose there exists $\alpha > 0$, $m>0$ such that $\M = C_\alpha^m(\M_g \cup \M_s)$ (Definition \ref{defn:path}). Then for any measurable disjoint sets $M_s$ and $M_g$, there exists a unique energy minimizing voltage $v^*$ (Definition \ref{defn:energy_minimizing_voltage}) over $M$ with $M_s$ as its source and $M_g$ as its sink. 
\end{proposition}

Given a unique energy-minimizing voltage, we can then define the effective resistance between $M_s$ and $M_g$ as the inverse of the induced current between them by $v^*$. 

\begin{definition}[Effective resistance on metric space]\label{prop:ER_metric_space}
Let $M_s, M_g$ be non-empty measurable disjoint subsets of $M$ such that they have a unique energy minimizing voltage, $v^*$, with respect to measure $\mu$ and kernel $k$. Then their effective resistance is defined as $R^{\mu}(M_s, M_g) = 1/J_{tot}$, where 
\begin{equation*}
    J_{tot} = \int_{M_s}\int_{M} k(x, y)(v^*(x) - v^*(y)) d\mu(x) d\mu(y).
\end{equation*}
\end{definition}

\section{Convergence towards the effective resistance} \label{section:convergence_of_ER}

In this section, we show how the effective resistance between two subsets of metric space $M$, $M_s$, and $M_g$, can be approximated by computing the effective resistance of a graph induced by points sampled from the probability measure, $\mu$, over $M$. To this end, let $X_n = \{x_1, x_2, \dots, x_n\} \sim \mu$ be a set of points sampled from data distribution $\mu$ over $M$. Using one of the kernels defined above, a weighted graph can be constructed on the samples by assigning weights $W_{ij} \propto k(\x_i, \x_j)$ between each pair of points $\x_i, \x_j \in X_n$. 

\begin{definition}[Metric resistor graph]\label{def:metric_resistor_graph}
Let $X_n$ a sample as defined above and $k: M \times M \to [0, 1]$ be a kernel similarity function. Then the \textbf{metric resistor graph}, $G_n = (X_{n}, W)$, is the weighted graph with edge weights $W_{ij} = \frac{k(x_i, x_j)}{n^2}$. 
\end{definition}


\noindent In the next section, we take a closer look at the scaling factor $1/n^2$.

\subsection{Scaling}\label{section:scaling}
Our goal is to construct a definition of the effective resistance that converges towards a non-trivial solution as the number of sampled points, $n$, goes towards infinity. To achieve this, it is necessary that the edge weights $W_{ij}$ scale appropriately with the number of samples. It is natural to demand that the physical properties of the graph, embodied by the resistance, current, and voltage, should remain relatively stable as $n$ increases. Thus, it is crucial to understand how the edge resistances should scale with the number of points sampled.

\paragraph{Intuition}
To this end, consider two small regions $T$ and $T'$ such that $k(x,x')>0$ for all $x\in T$ and $x'\in T'$.  For simplicity, let $k(x,x')$ be constant, which means each edge has equal resistance $R = 1/k(x,x')$.  We aim to keep the resistance between these two regions constant as the number of points changes.  For a fixed $X_n$, on average there are $m$ points $x\in S_n = \curlyb{x\in X_n : x\in T}$ and $m$ points $x'\in S_n' = \curlyb{x\in X_n : x\in T'}$.  This results in $m^2$ edges between $T$ and $T'$. This means the total resistance between these regions is $R/m^2$, given that these edges are connected in parallel.

The issue here is that, once we move to a denser sample $X_{2n}$, there will be, on average, $2m$ points in $S_{2n}$ and $S_{2n}'$ respectively.  This will create a net resistance $\frac{1}{4} R/m^2$, which means the resistance between these physical regions $T, T'$ decreases and will go to 0 as $n$ goes to infinity.
We illustrate this construction in Figure \ref{fig:resistor_split}.

\begin{figure}[htb!]
\centering
\includegraphics[width=0.9\textwidth]{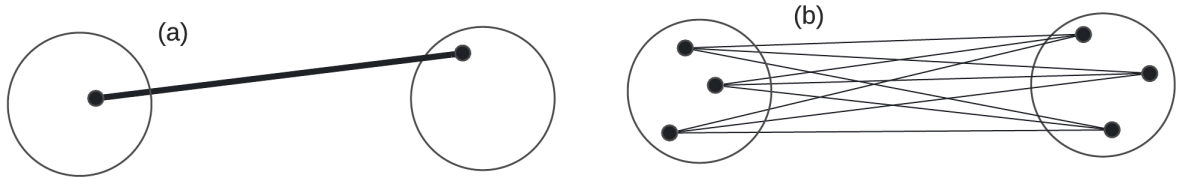}
\caption{Example of resistance scaling.  {\bf (a)} Number of edges connecting $T$ and $T'$ for $m=1$ point sampled from each region.  {\bf (b)} Number of edges connecting $T$ and $T'$ for $m=3$ points sampled from each region. Notice that the number of edges between these regions is $m^2$. } \label{fig:resistor_split}
\end{figure} 

\paragraph{Point-wise scaling} In order to overcome this issue of decreasing net resistance, we introduce a point-wise scaling $\gamma_{ij} = 1/n^2$ for all node pairs $i,j$, to compensate for the $m^2$ factor. Here $m = p n$ where $p$ is the probability that a sample drawn from $X_n$ falls in region $T$. In other words, with the point-wise scaling, we have the net resistance $\frac{1}{4}R/p p^\prime$, where $p,p^\prime$ are constant with respect to $n$.



Having defined the proper scaling, we now define the random object, $v_n^*$, which is the energy-minimizing voltage induced by a metric graph constructed from an i.i.d sample of $n$ points from $M$. 

\begin{definition}[Energy minimizing voltage over metric graph]\label{defn:finite_sample_voltage}
Let $M_s, M_g$ be disjoint measurable sets in $M$. For $X_n \sim \mu^n$, let $v_{X_n}: X_n \to [0, 1]$ be defined as the energy minimizing voltage over the metric graph $G_n$ where the weights $W_{ij}$ are defined by $W_{ij} = \frac{1}{n^2}k(x_i, x_j)$. Then $v_n^*: M \to [0, 1]$ denotes the function $$v_n^*(x) = \begin{cases} 1 & x \in M_s \\ 0 & x \in M_g \\ \frac{\sum_{x_i \in X_n} k(x, x_i)v_{X_n}(x_i)}{\sum_{x_i \in X_n} k(x, x_i)} & x \in M \setminus (M_s \cup M_g)\end{cases}.$$ Note that $v_n^*$ represents a random variable over functions $M \to [0, 1]$ with randomness induced by the randomness of $X_n$. 
\end{definition}

We similarly define the region-based ER, $R_n(M_s, M_g)$ as follows.

\begin{definition}[Region-based effective resistance]\label{defn:finite_sample_ER}
Let $M_s, M_g$ be disjoint measurable sets in $M$. For $X_n \sim \mu^n$, let $v^*_{n}$ be defined as the energy minimizing voltage over the metric graph, $G_n$ where the weights $W_{ij}$ are defined by $W_{ij} = \frac{1}{n^2}k(x_i, x_j)$. We then define the region-based ER $R_n(M_s, M_g)=1/J_{tot}$ where $J_{tot}$ is as defined in Proposition \ref{prop:dirichlet_formulation_ER_sets} with the voltage $v^*_{n}$. Note that $R_n(M_s, M_g)$ corresponds to a random variable, denoted $R_{X_n}$, induced by the randomness of $X_n$. 
\end{definition}

\subsection{Convergence analysis}
We now show that the finite sample energy minimizing voltage, $v_n^*$, and the finite sample region-based ER, $R_n(M_s, M_g)$, converge towards the limit objects $v^*$ and $R^\mu(M_s, M_g)$.

\begin{theorem}\label{thm:convergence}
Let $M, k,$ and $\mu$ satisfy that there exists $m > 0$ and $\alpha > 0$ such that $\M = C_\alpha^m(\M_g \cup \M_s)$. Let $M_s, M_g$ be disjoint measurable subsets of $M$, and let $v_n^*, v^*, R_n(M_s M_g),$ and $R^{\mu}(M_s, M_g)$ be as defined in Definitions \ref{defn:finite_sample_voltage}, \ref{defn:energy_minimizing_voltage}, \ref{defn:finite_sample_ER}, and \ref{prop:ER_metric_space}. Then the following hold: 
\begin{enumerate}
	\item For any $x \in M$, the sequence $v^*_1(x), v^*_2(x), \dots$ converges to $v^*(x)$ in probability.
	\item The sequence $R_1(M_s, M_g), R_2(M_s, M_g), \dots $ converges to $R^\mu(M_s, M_g)$ in probability. 
\end{enumerate}
\end{theorem}

A proof can be found in Appendix B. 

\section{Properties of the effective resistance between sets}\label{section:properties_of_regionbasedER}
We have established that the finite sample region-based ER $R_n(M_s, M_g)$ from Definition \ref{defn:finite_sample_ER} converges to the ER on a metric space $R^\mu(M_s, M_g)$, defined in Definition \ref{prop:ER_metric_space}. In this section, we establish that $R_n(M_s, M_g)$ is a distance metric. In particular, we prove that it satisfies the triangle inequality.

Consider a metric space $M$ and a sample $X_n\sim \mu^n(M)$. For two disjoint measurable subsets $M_s, M_g \subset M$ we can denote the corresponding subsets of $X_n$ as $X_s = \{x\in X_n : x\in M_s\}$ and $X_g = \{x\in X_n : x\in M_g\}$. The region-based ER $R_n(X_s, X_g)$ over a finite sample can be thought of as the ER between sets, namely $R^s(X_s, X_g)$ as defined in Section \ref{section:ER_sets}. For the analysis in this section, this is the interpretation we will take.

In the following, we will introduce the subscript $G$ on $R_G^{s}(X_a, X_b)$ to indicate the graph over which the ER is calculated. From Lemma \ref{lemma:relating_r_G_with_rG_reduced} it follows that $R_G^{s}(X_a, X_b) = R_{G_{ab}}^{s}(a, b) = (e_1 - e_2)^\top L_{G_{ab}}^\dagger (e_1 - e_2)$ where $L_{G_{ab}}$ is the Laplacian on the reduced graph $G_{ab}$ defined in Definition \ref{def:reduced_graph}. Symmetry and positive semi-definiteness of $R_G^{s}(X_a, X_b)$ follows therefore from the classical result on the reduced graph $G_{ab}$
\cite{jorgensen2008operator,ghosh2008minimizing,klein1993resistance}. 

Meanwhile, the triangle inequality is established by Theorem \ref{thm:triangle_ineq_for_ER_sets}. 

\begin{theorem}(Triangle inequality for effective resistance between sets)\label{thm:triangle_ineq_for_ER_sets}
Consider a graph $G = (X, W)$ and the non-empty disjoint subsets $X_a, X_b, X_z \in X$. Let $R^s(X_p, X_q)$ be the ER between sets $X_p, X_q$ for $p,q \in \{a,b,z\}$. We then have
\begin{equation*}
    R^s(X_a, X_b) \leq R^s(X_a, X_z) + R^s(X_z, X_b)
\end{equation*}
\end{theorem}
\begin{proof}
Appendix \ref{appendix:proof_triangle_ineq_er_sets}.
\end{proof}

Having established the triangle inequality along with symmetry and positive semi-definiteness, it follows that the effective resistance between disjoint sets is a distance metric.



\section{Computational considerations}\label{section:computational_complexity_considerations}
In practice, calculating the effective resistance in the large graph limit is not computationally viable. This is because the computational cost directly increases with $n$, the number of sampled points. For large values of $n$, it can be relatively expensive to compute the effective resistance -- especially if we desire to do so for many pairs of regions $M_s, M_g$. To control the computational complexity of the calculation, we suggest building the graph on a partitioning of the data which we call an $\alpha$-cover $\CC$ to remove the dependency on $n$. Furthermore, in order to incorporate information about the density, we combine the $\alpha$-cover with a suitable scaling of the graph weights.

\subsection{Alpha-cover}\label{section:epsilon_cover_and_region_wise_scaling} Let $(M,d)$ be a metric space and conisder an $\alpha$-cover $\CC$ as defined in Definition \ref{def:epsilonCover}. The $\alpha$-cover is closely related to the concept of doubling dimension $\textsf{ddim}(M)$ as defined in Definition \ref{def:DoublingDimension}, which is a measure of the intrinsic dimension of $M$. In fact, if $\CB(\eta)$ is the smallest ball such that $M \subseteq \CB(\eta)$ and $\alpha= 2^{-\ell}\eta$, then it follows that $|\CC| = 2^{\ell \textsf{ddim}(M)}$. This means that the size of the $\alpha$-cover depends only on the resolution, span of the data, and the intrinsic dimension of $M$, while it is independent of $n$.

 
\begin{definition}[$\alpha$-cover]
Let $(M,d)$ be a metric space. A subset $\CC \subseteq M$ is an $\alpha$-cover of $M$ if it satisfies the following two conditions. 
\begin{enumerate}
    \item (Packing property): All points $y_1,y_2 \in \CC$ are such that $d(y_1, y_2) \geq \alpha$ 
    \item (Covering property): For all points $x \in M$ there exists $y \in \CC$ such that $d(x,y) \leq \alpha$
\end{enumerate}
\label{def:epsilonCover}
\end{definition}

To construct the $\alpha$-cover, a popular algorithm is the cover-tree algorithm, see e.g.  Beygelzimer et al.~\cite{beygelzimer2006cover} and Oslandsbotn et al.~\cite{oslandsbotn2022Streamrak} (which is the algorithm we use in this paper). 

\paragraph{Region-wise scaling} 
In order to incorporate information about the underlying distribution $\mu(M)$, we combine the $\alpha$-cover with a region-wise scaling as an alternative to the scaling suggested in Section \ref{section:scaling}. This is important because, as mentioned in the introduction, one of the desirable properties of ER is that it can capture cluster structures in graphs, which in turn means information about the underlying data distribution $\mu(M)$. Because this information is lost with the $\alpha$-cover, due to the uniform partitioning, a scaling that captures the density is necessary.

To formally define this scaling, we first introduce the concept of a Voronoi cell
\begin{definition}[Voronoi cell]\label{def:voronoi_cell}
Let $\CC$ be an $\alpha$-cover as defined in Definition \ref{def:epsilonCover}. We define the Voronoi cell associated with point $x_i \in \CC$ as
    \begin{equation*}
         \CV_i = \{x\in M: d(x_i, x) \leq d(x_j, x)\, \forall x_i, x_j \in \CC\}
    \end{equation*}
\end{definition}

We can then define the scaling $\gamma_i$ for a given sample $X_n \sim \mu^n(M)$.
\begin{definition}[Region-wise scaling]\label{def:probability_scaling} Let $X_n$ be a sample as defined above and $\CC$ the associated $\alpha$-cover. With each center, $x_i\in \CC$, we associate a scaling constant
\begin{equation*}
    \gamma_i = \frac{|\{x \in X_n: x\in \CV_i\}|}{n}
\end{equation*}
\end{definition}

The interpretation of $\gamma_i$ is that it is the empirical local probability density associated with each Voronoi cell in the $\alpha$-cover. It can easily be estimated by counting the number of samples $x\in X_n: x\in \CV_i$, divided by the total number of samples $n$.




\begin{remark}
We note that for the $\alpha$-cover approach, we can not use the point-wise scaling introduced in Section \ref{section:scaling}. To see this, consider the discussion in Section \ref{section:scaling}, which showed that the net resistance between two regions $T$ and $T^\prime$ is $\frac{1}{4}R/m^2$. For data sampled from $\mu$, it follows that $m \propto n$. However, for the $\alpha$-cover, we have instead $m \propto |\CC|$, and the size of the epsilon cover $|\CC|$ depends only on the doubling dimension of the data and not on $n$.
\end{remark}


\subsection{Effective resistance on $\alpha$-cover}

Our idea is to construct a graph on the $\alpha$-cover, instead of directly on $X_n$, and then define the ER on this graph instead. We call this new graph the Cover resistor graph.

\begin{definition}[Cover resistor graph]\label{def:cover_resistor_graph}
Let $\calC$ be an $\alpha$-cover, and let $X_n$ a sample as defined above and $k: M \times M \to [0, 1]$ be a kernel similarity function. Let $\gamma_i, \gamma_j$ be the region-wise scaling weights defined in Definition \ref{def:probability_scaling}. Then the \textbf{cover resistor graph}, $G_n^\calC = (\calC, W)$, is the weighted graph with edge weights $W_{ij} = \gamma_i\gamma_jk(x_i, x_j)$. 
\end{definition}

The cover resistor graph essentially uses the nodes of the cover as vertices and constructs weights based on the relative \textit{weights} of each cover node.

We now show that computing effective resistance over the cover resistor graphs converges to the same limit object that using the sampled metric graphs in the previous section does. To do so, we begin by defining analogs of $v^*_n$ and $R_n(M_s, M_g)$.

\begin{definition}[Energy minimizing voltage over $\alpha$-cover]\label{def:energy_min_voltage_alpha_cover}
Let $\calC$ be an $\alpha$-cover, and $G_n^\calC$ be its cover resistor graph constructed from sample $X_n$, source and sink regions $M_s$ and $M_g$, and kernel function $k$. Let $v_{X_n, \calC}: \calC \to [0, 1]$ denote the energy minimzing voltage function over $G_n^\calC$. Then we let $v_n^\calC: M \to [0, 1]$ be defined as $$v_n^\calC(x) = \begin{cases} 1 & x \in M_s \\ 0 & x \in M_g \\ \frac{\sum_{c_i \in \calC} k(x, c_i)v_{X_n}(x_i)}{\sum_{c_i \in \calC} k(x, c_i)} & x \in M \setminus (M_s \cup M_g)\end{cases}.$$ Note that $v_n^\calC$ represents a random variable over functions $M \to [0, 1]$ with randomness induced by the randomness of $X_n$. 
\end{definition}

\begin{definition}[Region-based ER on $\alpha$-cover]\label{defn:finite_sample_ER_alpha}
Let $\calC$ be an $\alpha$-cover, let $v^{\CC}_n$ be the associated energy minimizing voltage and let $G^{\CC}_n$ be the cover resistor graph from Definition \ref{def:cover_resistor_graph} constructed from sample $X_n \sim \mu^n$. Consider the source and sink regions $M_s$ and $M_g$. We define the region-based ER on the $\alpha$-cover as $R_n^\calC(M_s, M_g) = 1/J_{tot}$ where $J_{tot}$ is as defined in Proposition \ref{prop:dirichlet_formulation_ER_sets} but now with respect the sets $\calC \cap M_s, \calC \cap M_g$, the graph $G^{\CC}_n$ and the voltage $v^{\CC}_n$. Note that $R_n^\calC(M_s, M_g)$ denotes the random variable, $R_{X_n}^\calC$, induced by the randomness of $X_n$. 
\end{definition}

We now show that similarly to $v_n^*$ and $R_n(M_s, M_g)$, for sufficiently small values of $\alpha$, $v_n^\calC$ and $R_n^\calC(M_s, M_g)$ converge to the same quantities. 

\begin{theorem}\label{thm:alpha_cover_converge}
Let $M$ be a metric space, $k$ a kernal similarity function, and $\mu$ be a measure over $M$. Let $M_s$ and $M_g$ be two disjoint measurable subsets of $M$ such that $M = C_\beta^m(M_s \cup M_g)$ for some $m, \beta > 0$. Then there exists a function $\Delta: \R^+ \to \R^+$ such that the following properties hold. First, $\lim_{\alpha \to 0^+} \Delta(\alpha) = 0$. Second, for any $\alpha > 0$ and any $\alpha$-cover of $M$, $\calC$, the following two conditions hold. 
\begin{enumerate}
	\item For any $x \in M$, the sequence $v_1^\calC(x), v_2^\calC(x), \dots$ converges in probability, and satisfies that $$|\lim_{n \to \infty} v_n^\calC(x) - v^*(x)| < \Delta(\alpha).$$
	\item The sequence $R_1^\calC(M_s, M_g), R_2^\calC(M_s, M_g), \dots$ converges in probability and satisfies that $$|\lim_{n \to \infty} R_n^\calC(M_s, M_g) - R^\mu(M_s, M_g)| < \Delta(\alpha).$$
\end{enumerate}
\end{theorem}

This result implies that for a sufficiently small value of $\alpha$, we can essentially replace our sample $X_n$ with any $\alpha$-cover, $\calC$, of $M$.

\subsection{A note on the advantages of using an $\alpha$-cover} The advantages of defining the ER on an $\alpha$-cover are two-fold:
\begin{enumerate}
    \item  With the $\alpha$-cover, the size of the graph and, therefore, the computational complexity of computing the ER can be controlled independently of $n$.
    \item The approximation of the density can be refined in a continuous manner by updating the local probabilities using more samples from $\mu(M)$, without the need to change the size of the graph.
\end{enumerate}

For example, with $\alpha = 2^{-l}\eta$, it follows that $|\CC|$ and, therefore, the graph size grows as $\CO((\eta/\alpha)^{\textsf{ddim}(M)})$. This means that the size of the $\alpha$-cover depends only on the resolution we want $\alpha$, the span of the data $\eta$, and the doubling dimension $\textsf{ddim}(M)$ of $M$ (Definition \ref{def:DoublingDimension}). 

We note that the $\alpha$-cover, with region-wise scaling, satisfies the criteria for a streaming algorithm \cite{muthukrishnan2005datastreams}, allowing for continuous refinement of the graph weights, without increasing the size of the graph.
The problem is that calculating the effective resistance between sets requires solving the Schur complement with respect to each set, which is computationally expensive; see Section \ref{section:schur_comp_formulation_er}.

\begin{remark}
    We note that the use of $\alpha$-cover and region-wise scaling is not restricted to the region-based ER proposed in this paper. It can also be used for computing the standard ER.
\end{remark}


\section{Experiments}\label{section:experiments}
In this section, we demonstrate the region-based effective resistance (region-based ER) in the large graph limit and show that it converges to a meaningful limit. We divide the section into three sets of experiments. In Section \ref{section:non-trivial_limit}, we replicate some of the experiments conducted in Von Luxburg et al.~\cite{von2010getting} and show that the region-based ER does not suffer from the trivial limit issue that standard ER suffer from. In Section \ref{section:Meaningful_limit} we take this a step further and demonstrate the convergence of region-based ER to a meaningful limit. Finally, in Section \ref{section:experiment_eps_cov}, we demonstrate a computationally efficient way to extend the calculations to large graphs in a controlled manner.



We note that to calculate the region-based ER, we utilize the Schur complement of the graph Laplacian with respect to these sets; see Appendix \ref{section:schur_comp_formulation_er} for more details. 

\paragraph{Setting:} Throughout this section, we consider standard ER and region-based ER on data sets $X_n\sim \mu^n(M)$, sampled from a distribution over a metric space $(M,d)$. For the standard ER we consider a resistor graph as described in Section \ref{section:resistor_graphs}. For the region-based ER we use a metric graph as defined in Definition \ref{def:metric_resistor_graph}. The standard ER between two nodes $x_i,x_j$ in the graph is denoted $R_{ij} \coloneqq R(x_i, x_j)$. For the region-based ER, we introduce the notation $R^s_{ij} \coloneqq R_n(X_i, X_j)$. We use a radial kernel to determine the sets $X_j, X_i$ associated with $x_i, x_j$. Namely $X_i = \{x\in X_n : \mathbbm{1}(d(x, x_i) \leq r)\}$ and similarly for $X_j$. Here $r_s$ is the source radius (which varies for each experiment).


\subsection{Region-based ER does not converge to the Von-Luxburg limit}\label{section:non-trivial_limit}
In this section, we replicate some of the experiments conducted in  Von Luxburg et al.~\cite{von2010getting}. For the experiments, we calculate both the standard ER and the region-based ER proposed in this paper. We show that where the standard definition of ER converges to a trivial limit as shown by Von Luxburg et al.~\cite{von2010getting}, the region-based ER does not.

We are interested in the convergence of the region-based ER compared to the standard ER as the number of samples used to construct the graph increases. We construct the graph on a sample $X_n$. Let $\eta_{ij} \coloneqq 1/D_i + 1/D_j$ be the Von-Luxburg limit for the standard ER; similarly, let $\eta^s_{ij} \coloneqq 1/D^s_i + 1/D^s_j$ be the  Von-Luxburg limit for the region-based ER. Here $D_i, D_j$ are the degrees of nodes $i,j$, respectively. Similarly, $D^s_i, D^s_j$ are the degrees associated with the sets $X_i, X_j$, see Appendix \ref{section:reduced_graph} and Lemma \ref{lemma:interp_reduced_graph} for more details. We then consider the max and mean of the relative deviation from the Von-Luxburg limit, namely:

\begin{equation}\label{eq:rel_dev_von_lux_standardER}
    \max_{ij}|R_{ij} - \eta_{ij}|/R_{ij} \quad \text{and} \quad  \sMean[|R_{ij} - \eta_{ij}|/R_{ij}].
\end{equation}
and
\begin{equation}\label{eq:rel_dev_von_lux_regionbasedER}
    \max_{ij}|R^s_{ij} - \eta^s_{ij}|/R^s_{ij} \quad \text{and} \quad   \sMean[|R^s_{ij} - \eta^s_{ij}|/R^s_{ij}].
\end{equation}

We consider the convergence of the relative deviation from the Von-Luxburg limit on two data sets that are similar to those studied in \cite{von2010getting}. These are:
\begin{itemize}[label=$\diamond$]
    \item Uniform 3-dim domain
    \item USPS data-set of handwritten digits (roughly $9200$ samples in $256$ dimensions)
\end{itemize}

\begin{figure}[h!]
    \centering
    \subfloat[\label{subfig:uniform_vonlux_region_False}\protect\centering Uniform domain/Standard ER]{\includegraphics[width=0.44\textwidth]{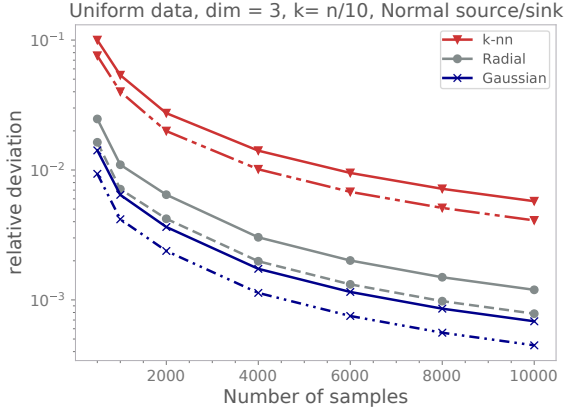} }
        \hfill
    \subfloat[\label{subfig:uniform_vonlux_region_True}\protect\centering Uniform domain/Region-based ER]{\includegraphics[width=0.44\textwidth]{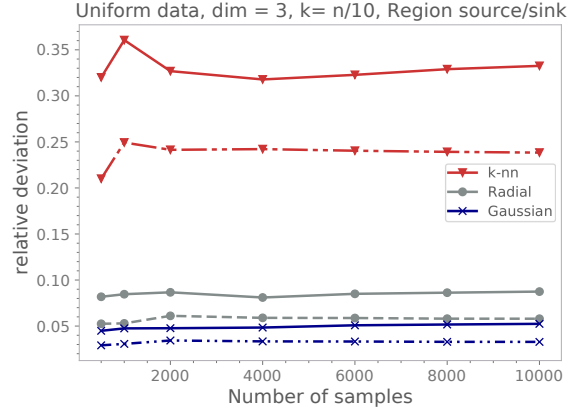} }
        \hfill
    \subfloat[\label{subfig:Usps_vonlux_region_False}\protect\centering USPS data set/Standard ER]{\includegraphics[width=0.44\textwidth]{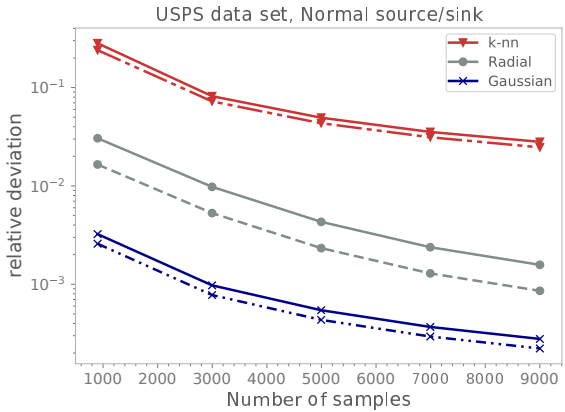} }
        \hfill
    \subfloat[\label{subfig:Usps_vonlux_region_True}\protect\centering USPS data set/Region-based ER]{\includegraphics[width=0.44\textwidth]{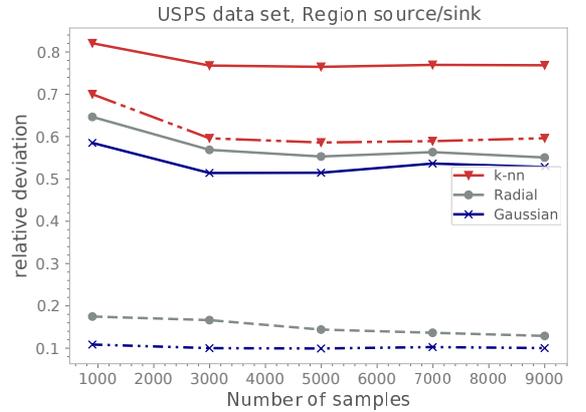} }
        \hfill
    \caption{This figure demonstrates that standard ER converges to the Von-Luxburg limit, while region-based ER does not (which is a good thing). Here solid lines show maximal relative deviations, and dashed lines show mean relative deviations. The x-axis shows the number of samples used to construct the graph. The y-axis is the relative deviation of the effective resistance from the Von-Luxburg limit.} 
    \label{fig:von_lux_exp}%
\end{figure}

On each data set, we build a graph using the kernels outlined in section \ref{section:resistor_graphs_in_metric_spaces} and also include the nearest neighbor kernel, $\Gamma(x,y) = \mathbbm{1}(y\in \CN_\kappa(x))\vee (x\in \CN_\kappa(y))$ to better replicate the corresponding experiments in Von Luxburg et al.~\cite{von2010getting}. Similarly to these experiments, we also select the radius of the radial kernel and the bandwidth of the Gaussian kernel to be the maximal k-nn distance in the data. Note that for the USPS data set, we let $k=100$, and for the uniform data set, we let $k=n/100$ where $n$ is the number of samples. We let the source radius be the maximal 20-nn distance in the data. 

\paragraph{Uniform domain} The results for the uniform domain are shown in Figure \ref{subfig:uniform_vonlux_region_False}-\ref{subfig:uniform_vonlux_region_True}. The figure shows the max and mean relative deviation from the Von-Luxburg limit for the standard ER (Figure \ref{subfig:uniform_vonlux_region_False}) and the region-based ER (Figure \ref{subfig:uniform_vonlux_region_True}). We see that the standard ER converges quickly to the Von-Luxburg limit, which corresponds to the results in Von Luxburg et al.~\cite{von2010getting}. Meanwhile, we see that for the region-based ER, convergence to the Von-Luxburg limit does not occur, which confirms our theoretical results. 

\paragraph{USPS data set} A similar set of experiments is shown for the USPS data set in Figure \ref{subfig:Usps_vonlux_region_False}-\ref{subfig:Usps_vonlux_region_True}. We observe the same behavior as for the uniform domain. The standard ER converges quickly to the Von-Luxburg limit, while the region-based ER does not converge to this limit. Again, this confirms our theoretical findings with respect to the region-based ER.

\subsection{Meaningful limit}\label{section:Meaningful_limit}
In the previous section, we saw that the region-based ER does not converge to the Von-Luxburg limit as was the case for the standard ER. However, it remains to be shown that the effective resistance under this new definition converges towards a meaningful limit. In this section, we take a closer look at this by considering the following two experiments
\begin{itemize}[label=$\diamond$]
    \item Convergence to a meaningful limit on a half-moon of increasing density
    \item Meaningful ordering of points on a Swiss roll.
\end{itemize}


\paragraph{Half-moon experiment} We consider a data distribution that consists of a background $[0, 1]^2$ with low density ($10000$ samples) and a half-moon with increasing density (samples in the interval $[100, \dots, 16000]$). This point cloud is illustrated in Figure \ref{subfig:halfmoon_data_factor2}.

The purpose of this experiment is to demonstrate that the region-based ER converges to the "geodesic distance" along the half-moon (see the grey dotted line in Figure \ref{subfig:halfmoon_data_factor2}), as the number of samples on the half-moon increases. This is a meaningful limit since the ER-based distance should consider all paths, and as the density of the half-moon becomes increasingly dominant, the distance should converge to the distance of paths along this curve.

\begin{figure}[htb!]
    \centering
    \subfloat[\label{subfig:halfmoon_data_factor2}\protect\centering]{\includegraphics[width=0.48\textwidth]{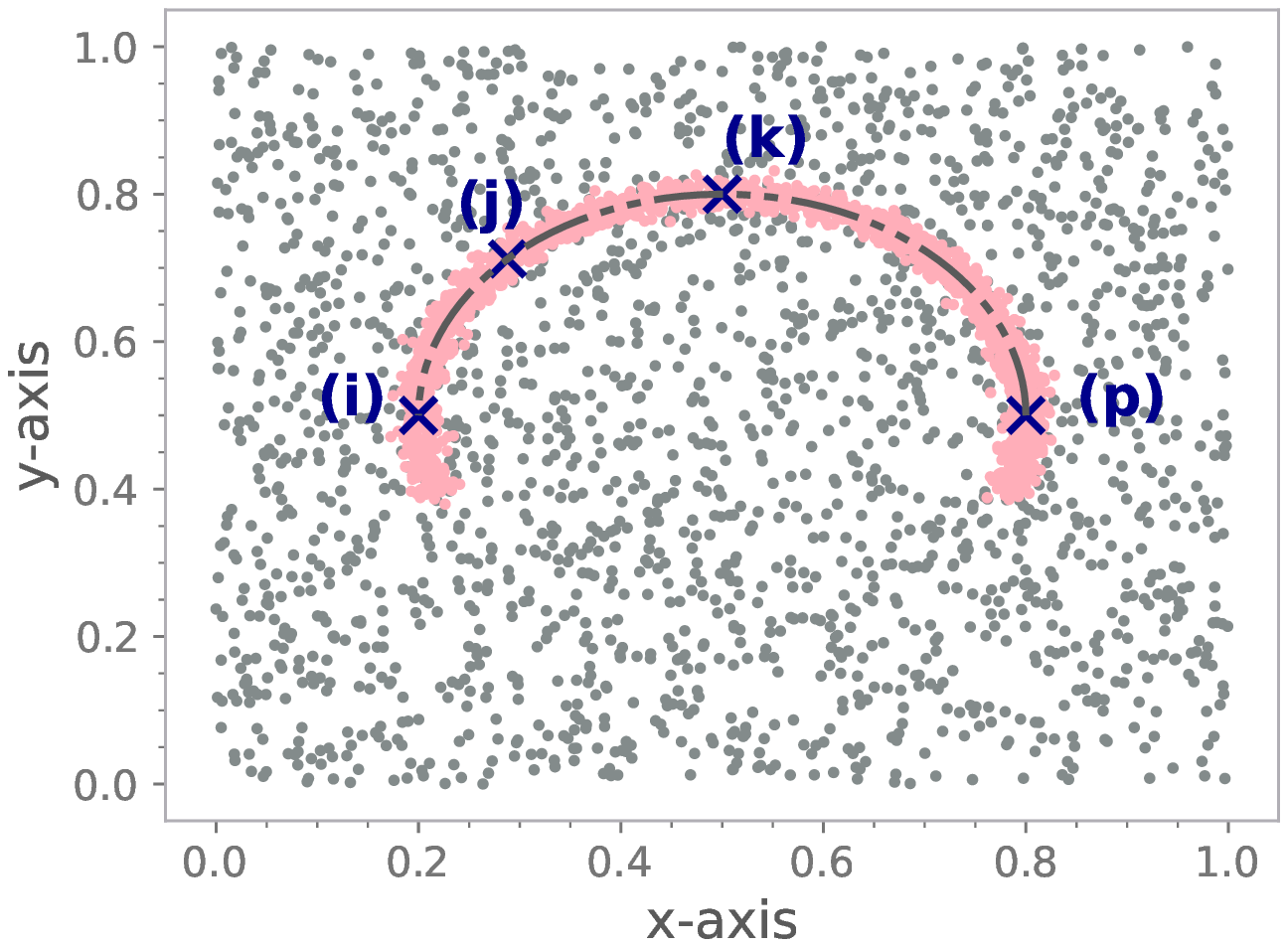} }
        \hfill
    \subfloat[\label{subfig:ratio_ers_factor2.eps}\protect\centering]{\includegraphics[width=0.44\textwidth]{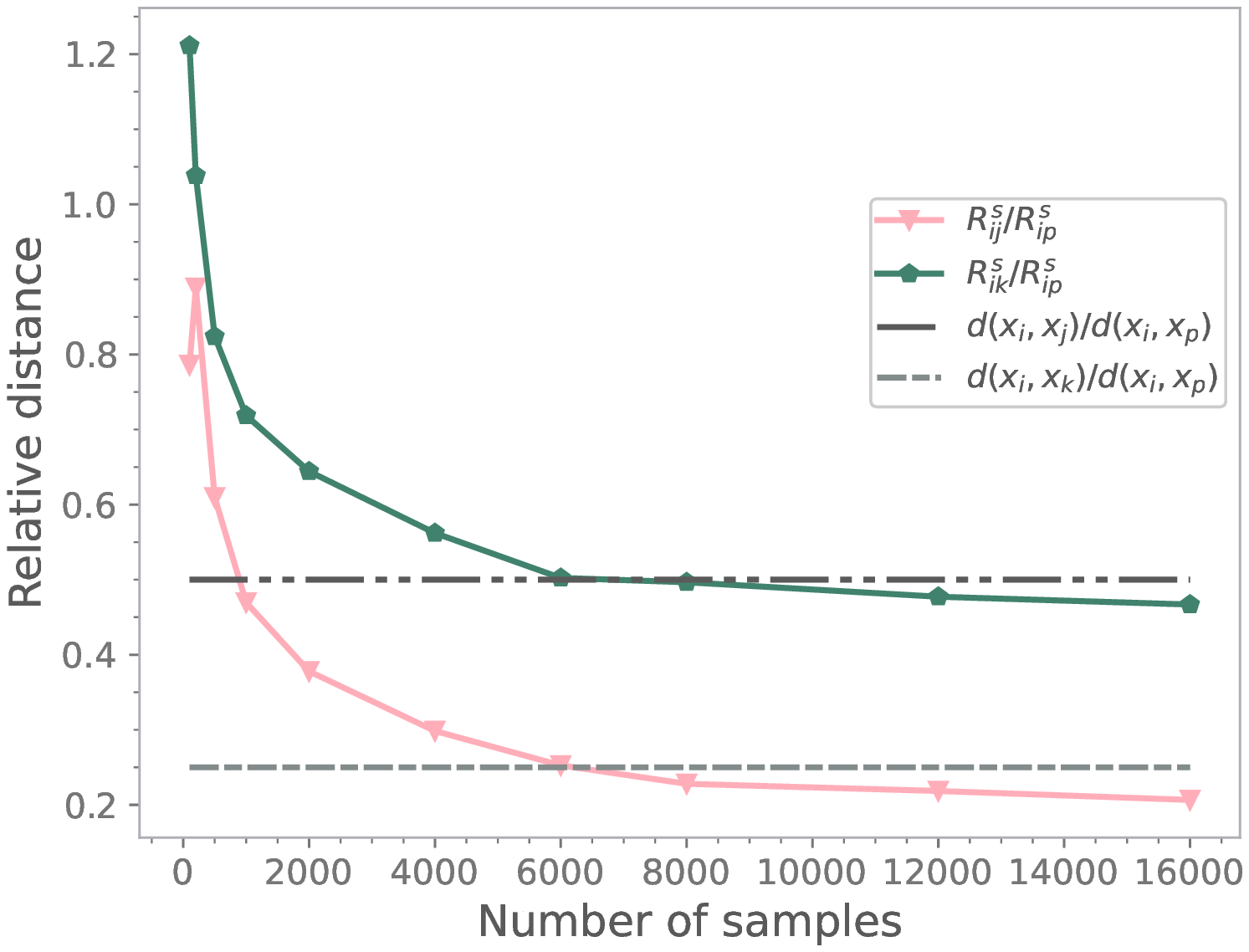} }
        \hfill
    \caption{The region-based ER between points on the half-moon converges to the distance along the half-moon. (This is what we want to see). (a) High-density half-moon (pink) over low-density background (grey). The points we consider are labeled $i,j,k$, and $p$. (b) Grey dotted lines shows $\Gamma_{ijp}$ and $\Gamma_{ikp}$ (See Eq. \eqref{eq:definition_of_halfmoon_geodesic_ratios}). Pink line shows $R^s_{ij}/R^s_{ip}$ and Green line shows $R^s_{ik}/R^s_{ip}$.} 
    \label{fig:Meaningfull_limit_examples}%
\end{figure}

The half-moon we consider is sampled from a circle segment with radius $t=0.3$ and angle $\theta\in[-20, 200]^\circ$ with a Gaussian distribution $\CN((t, \theta), 0.01)$ for each $\theta$. Along the half-moon, we consider four points $x_i, x_j, x_k$, and $x_p$; placed respectively at $\theta_i = 0^\circ$, $\theta_j = 45^\circ$, $\theta_k = 90^\circ$, and $\theta_p = 180^\circ$. We construct a graph on the point cloud using a radial kernel with radius $r=0.08$. For the source and sink regions, we use a source radius $r_s = 0.05$ centered on each source-sink point, respectively. 

Let $d(\cdot,\cdot)$ denote the distance along the half-moon arch as illustrated in Figure \ref{subfig:halfmoon_data_factor2} by the grey dotted line. In order to compare the region-based ER distance with $d(\cdot,\cdot)$ we need to consider these distances relative to some reference distance. Because of this, we introduce the distance between points $i$ and $p$ as the reference. We then have
\begin{equation}\label{eq:definition_of_halfmoon_geodesic_ratios}
    \Gamma_{ijp} \coloneqq d(x_i, x_j)/d(x_i, x_p) = 0.25 \quad \text{and} \quad \Gamma_{ikp} \coloneqq d(x_i, x_k)/d(x_i, x_p) = 0.5
\end{equation}

Figure \ref{subfig:ratio_ers_factor2.eps} shows the region-based ER ratios $R^s_{ij}/R^s_{ip}$ and $R^s_{ik}/R^s_{ip}$ as the number of points sampled from the half moon increases. The ratios $R^s_{ij}/R^s_{ip}$ and $R^s_{ik}/R^s_{ip}$ converges towards $\Gamma_{ijp}$ and $\Gamma_{ikp}$ respectively. This is expected because, as the density on the half-moon increases, the effective resistance, which considers all possible paths, should be increasingly dominated by the paths along the half-moon. 

\begin{remark}\label{remark:halfmoon_remark}
We note that had the region-based ER converged to the Von-Luxburg limit, we would not have observed the convergence in Figure \ref{subfig:ratio_ers_factor2.eps}. This is because the Von-Luxburg limit only depends on the degrees of the respective sets, which for the half-moon would be the same for all points $i,j,k,p$. Therefore, one would expect $R^s_{ij}/R^s_{ip}$ and $R^s_{ik}/R^s_{ip}$ to converge to $1$ if this was the case.
\end{remark}

\paragraph{Swiss roll experiment} We consider a data distribution shaped as a Swill roll and compare the relative distance between five points along the Swiss roll surface as indicated in Figure \ref{subfig:swiss_roll_3D_illustration_ink}. We consider the source radius $r_s=0.1$ centered at the point and use a radial kernel with radius $r=0.2$ to construct the graph. 

\begin{figure}[htb!]
    \centering
    \subfloat[\label{subfig:swiss_roll_3D_illustration_ink}\protect\centering ]{\includegraphics[width=0.48\textwidth]{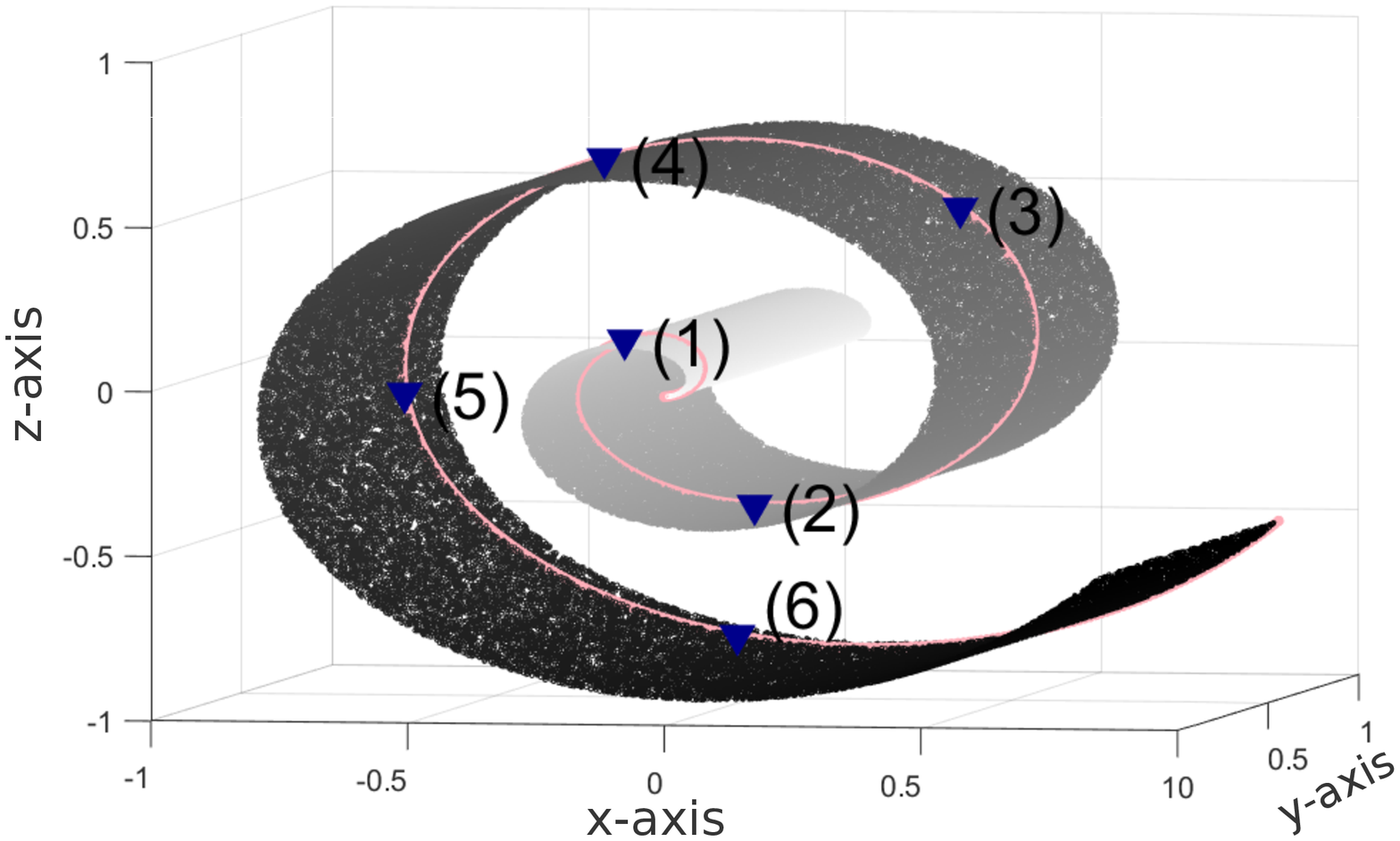} }
        \hfill
    \subfloat[\label{subfig:effres_scaling}\protect\centering 
    ]{\includegraphics[width=0.44\textwidth]{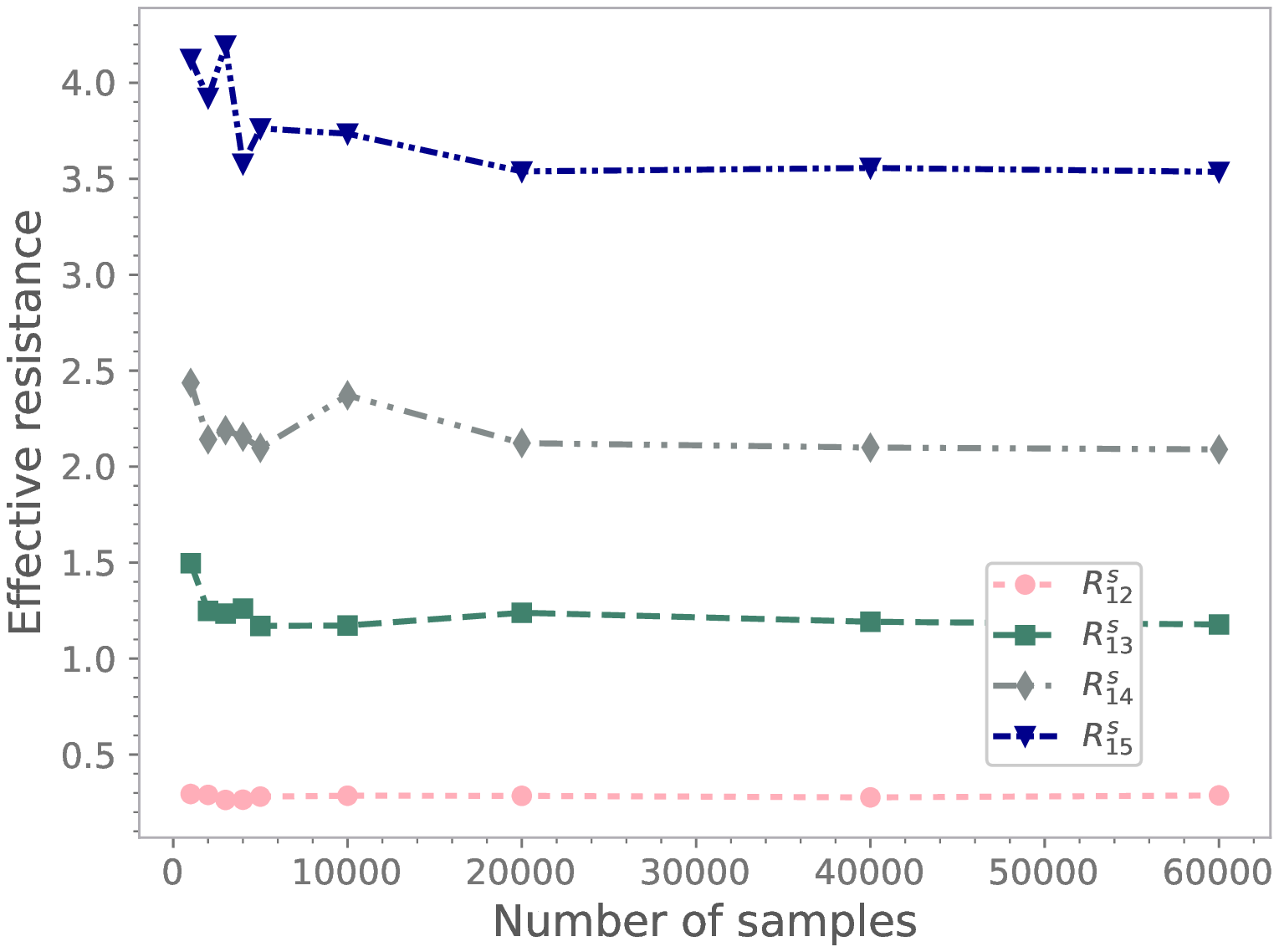} }
        \hfill
    \caption{The ordering of the lines in (b) is meaningful (which is a good thing). Experiment demonstrating that the region-based ER distance gives a meaningful ordering of the distance between points and maintains this as the number of samples used to construct the graph increases. (a) Data distribution. Blue triangles indicate the location of the points $1,\dots, 6$. (b) The y-axis shows the region-based ER scaled by a factor of $10^{-5}$. The x-axis shows the number of samples used to construct the graph. The pink line corresponds to the region-based ER $R^s_{12}$ between $1,2$. Similarly, for $R^s_{13}$ (green),  $R^s_{14}$ (gray), and $R^s_{15}$ (blue). } 
    \label{fig:er_convergence_on_halfmoon}%
\end{figure}

The region-based ER $R^s_{1i}$ between point $1$ and respectively points $i \in [2,3,4,5]$ is shown in Figure \ref{subfig:effres_scaling}. Using $R^s_{1i}$ as a measure of distance between the points, we see that region-based ER gives a natural ordering of the distance from $1$ to the other points, which is maintained as the number of samples increases. Namely, that $R^s_{12} < \dots <  R^s_{15}$.

\begin{remark}
We note that had the region-based ER converged to the Von-Luxburg limit, this ordering would not have been maintained for the same reason as discussed in the half-moon experiment; See remark \ref{remark:halfmoon_remark}. 
\end{remark}


\subsection{Example on the benefit of using an $\alpha$-cover graph}\label{section:experiment_eps_cov}
We include a simple experiment to demonstrate the use of an $\alpha$-cover and region-wise scaling, which was introduced in Section \ref{section:epsilon_cover_and_region_wise_scaling}. We consider data sampled from a non-uniform density $\mu([0, 1])$ consisting of two regions of high density separated by a region of low density (See Figure \ref{fig:data_distribution_nonuni1D}). We consider five points $\CJ = \{1,\dots, 5\}$ marked by blue triangles in Figure \ref{fig:data_distribution_nonuni1D} and calculate the region-based ER $R^s_{1j}$ between $1$ and points $j \in \CJ\backslash \{1\}$. In the experiment, the source radius is $r_s = 0.1$, and we use a radial kernel with radius $r=0.1$. 

We calculate the region-based ER using two different graphs; 
\begin{enumerate}[label=(\Alph*)]
    \item Graph built on an $\alpha$-cover with $\alpha=2/3 \times 3^{-6}$ (1122 centers) and with region-wise scaling
    \item Graph built with samples directly from $\mu([0, 1])$ and with point-wise scaling.
\end{enumerate}

\noindent We note that in order to construct the $\alpha$-cover, we use the cover-tree algorithm \cite{beygelzimer2006cover, oslandsbotn2022Streamrak}. 

\begin{figure}[h]
\includegraphics[width=\textwidth]{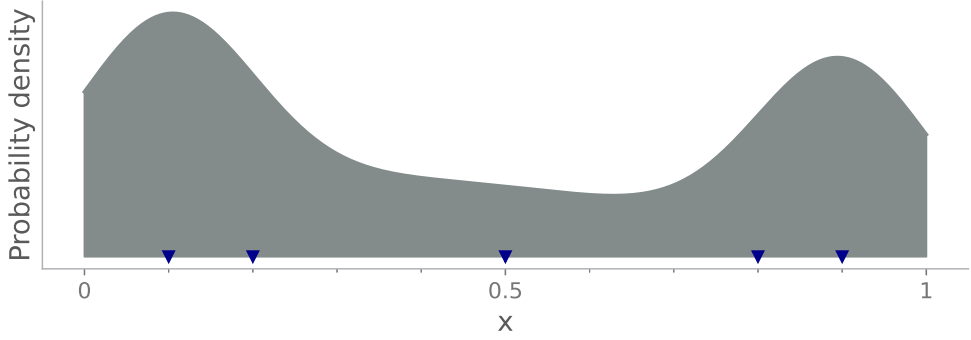}
 \caption{Non-uniform data distribution $\mu([0, 1])$. The y-axis shows the density of the data distribution  $\mu([0, 1])$. The x-axis shows the support. The points we consider are marked by the blue triangles labeled $(1), \dots, (5)$. } 
 \label{fig:data_distribution_nonuni1D}%
\end{figure}

In the experiment, the region-based ER $R^s_{1j}$ between point $1$ and points $j\in [2,3,4,5]$ are calculated for both the $\alpha$-cover graph with $p_i$ scaling (A) and the density graph (B). Figure \ref{fig:effres_epsilon_cover_experiment} shows that the region-based ER on the graph (A) converges to the same limit as the region-based ER on graph (B) when increasing the number of samples used to estimate the local probabilities $p_i$. 

This is what we wanted to show because it means that region-based ER can be calculated in two ways, either using a graph constructed on an $\alpha$-cover with appropriate scaling or by using a graph constructed directly on samples from $\mu([0, 1])$. The benefit of constructing the graph on the $\alpha$-cover is that the graph size will be independent of the number of samples. At the same time, accuracy can still be increased by improving the estimates of the local probabilities $p_i$. This has the desirable property that it satisfies the condition of a streaming algorithm; See Section \ref{section:computational_complexity_considerations} for more details.

\begin{figure}[htb!]
    \centering
    \subfloat[\label{subfig:effres_no_epscov}\protect\centering Graph built on samples from density]{\includegraphics[width=0.43\textwidth]{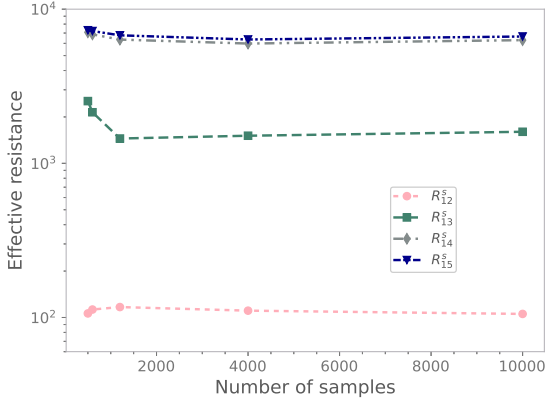} }
        \hfill
    \subfloat[\label{subfig:effres_fixed_epscov_density_scaling}\protect\centering Graph built on $\alpha$-cover with region-wise scaling]{\includegraphics[width=0.43\textwidth]{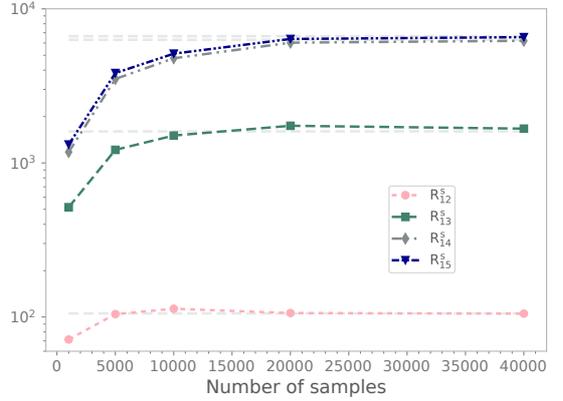} }
        \hfill
    
     \caption{Demonstration of region-based ER on $\alpha$-cover. We see that the ER in (b) converges to the asymptotics of the ER in (a) (weak gray lines). (This 
    is a good thing). The dotted lines corresponds to region-based ER $R^s_{1j}$ between point $1$ and points $j\in [2,3,4,5]$ respectively. The pink line corresponds to $R^s_{12}$, the green line to $R^s_{13}$, the grey line to $R^s_{14}$ and the blue line to $R^s_{15}$. The y-axis shows the region-based ER, while the x-axis is different for each sub-figure. (a) x-axis shows the number of samples used to construct the graph. (b) x-axis shows the number of samples used to estimate the local probabilities $p_i$ for scaling of the $\alpha$-cover graph.} 
    \label{fig:effres_epsilon_cover_experiment}%
\end{figure}

Furthermore, using an $\alpha$-cover graph with region-wise scaling, a smaller graph can be used to calculate the region-based ER, provided sufficient samples are used to estimate the local probabilities $p_i$. Since estimating these local probabilities is cheap, time is saved. Table \ref{table:time_example_epscov_graph} illustrates this; the table has to be seen in relation to the convergence results in Figure \ref{fig:effres_epsilon_cover_experiment}.

\begin{table}[!h]
\caption{Example of the time saved by calculating ER on an $\alpha$-cover graph (graph A) instead of a graph constructed directly on the samples from the distribution (graph B). The first column shows the graphs used to calculate the ER; These are Graphs (A) and (B), described earlier. The second column is the number of samples used (The number in parenthesis is the size of the $\alpha$-cover). The last column is the time to calculate the region-based ER using the Schur complement on the two graph types (The time in parenthesis is the time to estimate the local probabilities $p_i$ for the $\alpha$-cover).}\label{table:time_example_epscov_graph}
\begin{center}
\begin{tabular}{c|c|c} 
Graph type & Number of samples & Time (s)\\ \hline
Graph (A) & 21122 (1122) & 0.04s + (0.07s)\\  
Graph (B) & 4000 &  4.6s \\  
\end{tabular}
\end{center}
\end{table}

\begin{remark}
    An additional note is that also for this experiment, using the region-based ER as a measure of distance gives a meaningful ordering, namely, $R^s_{12} < R^s_{13} < R^s_{14} < R^s_{15}$. Moreover, as more samples are used to estimate the local probabilities, the distance between the samples increases. Due to the shape of the density, it makes sense that the distance between points should be larger when the density is taken into account. Especially points separated by the region of low density, which is also what we observe.

\end{remark}

\section*{Acknowledgments}
AC was partially funded by NSF DMS 1819222 and 2012266, and a gift from Intel research. YF is funded by the NIH grant NINDS (PHS) U19NS107466 Reverse Engineering the Brain Stem Circuits that Govern Exploratory Behavior. RB is funded by NSF under CNS 1804829. AO is funded by Simula Research Laboratory.

\bibliographystyle{unsrt}
\bibliography{References}

\newpage
\appendix
\renewcommand{\thesection}{\Alph{section}}

\section{Definitions}

\begin{definition}[The doubling dimension](Adapted from \cite{abraham2006advances})\\
Let $(M,d)$ be a metric space. The \textbf{doubling constant} of $M$ is the minimal $\kappa$ required to cover a ball $B_r(x)$ by $\kappa$ balls of radius $r/2$, for all $x\in M$ and for all $r >0$.
The \textbf{doubling dimension} of $X$ is defined as $\textsf{ddim}(M) = \log_2(\kappa)$.
\label{def:DoublingDimension}
\end{definition}

\subsection{Schur complement formulation of effective resistance between sets}\label{section:schur_comp_formulation_er}
\cite{song2019extension} offered an explicit expression for the effective resistance between sets $X_a, X_b$ in terms of the Schur complement. We restate this expression here, as we will use it when calculating the effective resistance for our experiments. Let $L/L_{cc}$ be the Schur complement, Definition \ref{def:schur_complement}, of the Laplacian $L$ with respect to the block $L_{cc}$, where $L_{cc}$ is the block corresponding to nodes in the set $X_c = X\backslash (X_a \cup X_b)$. The effective resistance between the sets $X_a, X_b$  can then be defined as
\begin{equation}\label{eq:eff_res_schur_def}
    R^s(X_a, X_b) = (e_{X_a}^\top (L/L_{cc}) e_{X_a})^{-1}
\end{equation}
where $e_{X_a}$ is the vector with all ones for $i\in X_a$ and zero otherwise.

\begin{definition}[Schur complement]\label{def:schur_complement}
Consider the block partition of the matrix $M=\begin{bmatrix}
    A & B \\ C, & D
\end{bmatrix}$, where $A\in \bbR^{n\times n}$, $B\in \bbR^{n\times m}$,
$C\in \bbR^{m\times n}$ and $C\in \bbR^{m\times m}$. Provided $D$ is non-singular, we define the Schur complement of $M$ as $M/D = A - BD^{-1}C$. 
\end{definition}

\section{Proof of Proposition \ref{prop:unique_voltage_solution}}

\subsection{Properties of Contractions}

In this section, we develop several tools that we will use throughout our entire proofs section. Let $Z$ denote a space. We will be interested in functions, $v: Z \to \reals$, as well as operators $A: v \mapsto (Av: Z \to \reals)$ that take functions to functions. For our purposes, $Z$ will often be $\M$, the underlying metric space, but will sometimes also be a finite sample from $\M$.

We begin by defining the $\ell_\infty$-norm, which will play a key role in our analysis. 
\begin{definition}\label{defn:infty_norm}
Let $v: Z \to \reals$ be a map. Then $||v||_\infty$ denotes the $\ell_\infty$ norm of $v$, and is defined as $||v||_\infty = \sup_{z \in Z} |v(z)|.$
\end{definition}

We can also define the $\ell_\infty$-norm of an operator.
\begin{definition}\label{defn:operator_norm}
Let $A$ be an operator, meaning it maps functions ($Z \to \reals$) to other functions. Then $$||A||_\infty = \sup_{||v||_\infty > 0, v: Z \to \reals} \frac{||Av||_\infty}{||v||_\infty}.$$ 
\end{definition}

We will be especially interested in \textit{contractions}, which are operators with infinity norm strictly less than $1$. We are also interested in contractions combined with translations. We call such operators nice.

\begin{definition}\label{defn:nice}
An operator, $T$, is \textbf{nice}, if there exists a contraction $A$ and a map $b$ such that $Tv = Av + b$. 
\end{definition}

We are now prepared for our first useful result:

\begin{lemma}\label{lem:fixed_point}
Let $T$ be a \textbf{nice} operator. Then there is a unique map $v: Z \to \reals$ such that $Tv(z) = v(z)$ for all $z \in Z$. 
\end{lemma}

\begin{proof}
Let $Tv = Av + b$. Define $u_0 = b$ and $u_i = Tu_{i-1}$ for $i \geq 1$. Let $||A||_\infty = \rho$ where $0 \leq \rho < 1$ because $T$ is nice. Observe that for $i \geq 1$, 
\begin{equation*}
\begin{split}
||u_{i+1} - u_i||_\infty &= ||(A u_i(x) + b) - (Au_{i-1} + b)||_\infty \\
&= \sup_{x \in \M} ||A(u_i - u_{i-1})||_\infty \\
&\leq \rho ||u_i - u_{i-1}||_\infty.
\end{split}
\end{equation*}
It follows that $u_0, u_1, \dots$ is a Cauchy sequence under the $\ell_\infty$ metric. Since the set of all functions on the reals is closed, it follows that $u_0, u_1, \dots, $ converges to some $u$, which must satisfy $Tu = u$. 

To show uniqueness, we can simply bound the infinity distance between any two fixed points to see that this distance is at most $\rho$ times itself. Since $\rho < 1$, it follows that the two fixed points must be the same function $u$.
\end{proof}

We prove one addition lemma about nice operators.

\begin{lemma}\label{lem:close_to_fixed_point}
Let $T$ be a nice operator with $Tv = Av + b$ where $||A||_\infty = \rho$. Suppose that function $v$ satisfies $||v - Tv||_\infty < \epsilon.$ If $u$ denotes the unique fixed point of $T$, then $$||u - v||_\infty < \frac{\epsilon}{1 - \rho}.$$
\end{lemma}

\begin{proof}
By using the same argument as the previous lemma, we see that the sequence $v, Tv, T^2v, \dots$ must converge to $u$. Since $||v - Tv||_\infty < \epsilon$, it follows that $||T^iv - T^{i+1}|| < \rho^i \epsilon.$ Summing the infinite geometric sequence gives us the desired result. 
\end{proof}

\subsection{Proving the existence of the energy-minimizing voltage (the limit object): Proposition \ref{prop:unique_voltage_solution}}

We begin by defining an operator that characterizes our desired limit object, $v^*$, (defined in Definition \ref{defn:energy_minimizing_voltage}). In Lemma \ref{lem:bound_convolve} - \ref{lemma:existence_of_fixpoint}, we prove the existence and uniqueness of $v^*$. In Lemma \ref{lemma:proof_that_vstar_is_minimizer}, we then prove that $v^*$ is the energy-minimizing voltage from Definition \ref{defn:energy_minimizing_voltage}.

\begin{definition}\label{defn:affine}
Let $k$ be a kernel and $\nk$ be the normalized version (Definition \ref{defn:norm_kernel}). Then $A_*$ is the operator defined as follows. If $v: \M \to \reals$ is a measurable function, then $A_*v$ is the function $\M \to \reals$ defined by $$(A_*v)(x) = \int_{\M} v^*(y)\nk(x, y)\ind(y \in \M \setminus (\M_g \cup \M_s))d\mu(y).$$ We also let $b_*: \M \to \reals$ denote the fixed function defined as $$b_*(x) = \int_{\M} \nk(x, y)\ind(y \in \M_s)d\mu(y).$$ Together, we let $T_*$ be the operator $$T_*v = A_*v + b_*.$$
\end{definition}

Our main idea will be to show the following two statements holds: 
\begin{enumerate}
	\item There exists $m> 0$ such that $T_*^m$ is contractive (Definition \ref{defn:nice}). 
	\item $v^*$ is a fixed point of $T_*$. 
\end{enumerate}

Together with Lemma \ref{lem:fixed_point}, this will prove the existence and uniqueness of $v^*$. We start by proving the first claim, which relies on the following technical Lemma that characterizes iterative powers of $A_*$. 

\begin{lemma}\label{lem:bound_convolve}
Let $v: \M \to \reals$ be a measurable map and $x \in \M$ be a point. Then for all $i \geq 1$, $$|(A_*^iv)(x)| \leq \int_\M |v(y)|\nk^{(i)}(x,y)\ind\left(y \in \M \setminus (\M_g \cup \M_s)\right).$$
\end{lemma}

\begin{proof}
We proceed by induction on $i$. In the base case, $i =1$, and we have 
\begin{equation*}
\begin{split}
|(A_*v)(x)| &= \left|\int_\M  v(y)\nk(x, y)\ind\left(y \in \M \setminus (\M_g \cup \M_s) \right)d\mu(y)\right| \\
&\leq  \int_\M  |v(y)|\nk(x, y)\ind\left(y \in \M \setminus (\M_g \cup \M_s) \right)d\mu(y) \\
&\leq \int_\M |v(y)| \nk(x,y)d\mu(y).
\end{split}
\end{equation*}
For the inductive step, let $i > 1$ and assume that the claim holds for $i-1$. Then
\begin{equation*}
\begin{split}
|(A_*^iv)(x)| &= |(A_*(A_*^{(i-1)}v))(x)| \\
&= \left|  \int_\M (A_*^{(i-1)}v)(y)\nk(x, y)\ind\left(y \in \M \setminus (\M_g \cup \M_s)\right)d\mu(y)\right| \\
&\leq \int_\M \left|(A_*^{(i-1)}v)(y) \right|\nk(x, y)\ind\left(y \in \M \setminus (\M_g \cup \M_s)\right)d\mu(y) \\
&\leq \int_\M \left|(A_*^{(i-1)}v)(y) \right|\nk(x, y)d\mu(y) \\
&\leq \int_\M \left(\int_\M |v(z)|\nk^{(i-1)}(y, z)\ind\left(z \in \M \setminus (\M_g \cup \M_s)\right)d\mu(z)\right) \nk(x, y) d\mu(y) \\
&= \int_\M \left(\int_\M \nk(x, y)\nk^{(i-1)}(y, z)d\mu(y)\right)  |v(z)|\ind\left(z \in \M \setminus (\M_g \cup \M_s)\right)d\mu(z) \\
&= \int_\M |v(z)|\nk^{(i)}(x, z)\ind\left(z \in \M \setminus (\M_g \cup \M_s)\right)d\mu(z),
\end{split}
\end{equation*}
as desired. Here the inequalities hold by
\begin{enumerate}
	\item Moving the absolute value into the expectation.
	\item bounding the indicator function by 1.
	\item Applying the inductive hypothesis.
	\item Applying Fubini's theorem to switch the order.
	\item Applying convolution (Definition \ref{defn:convolution}). 
\end{enumerate}
\end{proof}

We now show that there exists a power of $A_*$ that is a contraction.

\begin{lemma}\label{lem:contraction}
Let $\alpha$ be as in the statement of Proposition \ref{prop:unique_voltage_solution}. Then for all maps $v: \M \to \reals$, $$\sup_{x \in \M} |(A_*^mv)(x)| \leq (1-\alpha)\sup_{x \in \M}|v(x)|.$$
\end{lemma}

\begin{proof}
Fix $x \in \M$. Then by applying Lemma \ref{lem:bound_convolve},
\begin{equation*}
\begin{split}
|(A_*^mv)(x)| &\leq \int_\M |v(y)|\nk^{(m)}(x,y)\ind\left(y \in \M \setminus (\M_g \cup \M_s)\right)d\mu(y)\\
&\leq \sup_{x \in \M}|v(x)| \int_\M \nk^{(m)}(x,y)\ind\left(y \in \M \setminus (\M_g \cup \M_s)\right)d\mu(y).
\end{split}
\end{equation*}
Because $\nk$ is normalized, $\nk^{(m)}$ is as well, meaning that $\int_\M \nk^{(m)}(x, y)d\mu(y) = 1$. However, we also have by assumption (Proposition \ref{prop:unique_voltage_solution}) that $\M = C_\alpha^m(\M_g \cup \M_s)$, which implies that $$\int_\M \nk^{(m)}(x,y)\ind(y \in \M_g \cup \M_s)d\mu(y) \geq \alpha.$$ Substituting this, it follows that
\begin{equation*}
\begin{split}
|(A_*^mv)(x)| &\leq \sup_{x \in \M}|v(x)| \int_\M \nk^{(m)}(x,y)\ind\left(y \in \M \setminus (\M_g \cup \M_s)\right)d\mu(y). \\
&\leq (1-\alpha)\sup_{x \in \M}|v(x)|
\end{split}
\end{equation*}
\end{proof}

We are now prepared to show that there exists a unique fixed point of $T_*$. 

\begin{lemma}\label{lemma:existence_of_fixpoint}
There exists a unique function $v^*: M \to [0, 1]$ such that $T_*v^* = v^*$.
\end{lemma}

\begin{proof}
Let $m$ be as in Lemma \ref{lem:contraction}. Observe that the operator $T_*^mv$ can be written as $$T_*^m: v \mapsto A_*^mv + b_*(x) + (A_*b_*)(x) + \dots + (A_*^{m-1}b_*)(x).$$ By Lemma \ref{lem:contraction}, $A_*^m$ is a contraction and therefore $T_*^m$ is a nice (Definition \ref{defn:nice}) operator. It follows by Lemma \ref{lem:fixed_point} that $T_*^m$ has a unique fixed point which we denote as $v^*$. 

Any fixed point of $T_*$ is a fixed point of $T_*^m$, and therefore it suffices to show that $v^*$ is a fixed point of $T_*$ -- uniqueness will follow from the unqiueness of $v^*$ w.r.t. $T_*^m$. To do so, observe that $$T_*^m(T_*v^*) = T_*^m(T_*v^*) = T_*(T_*^m v^*) = T_*(v^*).$$ Thus $T_*v^*$ is a fixed point of $T_*^m$, meaning that is must equal $v^*$ by the uniqueness of $v^*$. Thus $T_*v^* = v^*$, as desired. 
\end{proof}

We now show that $v^*$ is the energy-minimizing voltage of the energy defined in Definition \ref{defn:energy_minimizing_voltage}.

\begin{lemma}\label{lemma:proof_that_vstar_is_minimizer}
Let $M_s, M_g \subset M$ be measurable disjoint subsets. Let $V_{M_s, M_g}$ be the set of all measurable functions $v: M \to [0, 1]$ with $v(x) = 1$ for all $x \in M_s$ and $v(x) = 0$ for all $x \in M_g$. The minimizer of the energy $E(v)$ defined in Definition \ref{defn:energy_minimizing_voltage}, is the function $v\in V_{M_s, M_g}$ that satisfies $v^* = T_* v^*$.
\end{lemma}
\begin{proof}
To show that $v^*$ is a minimizer of $E(v)$ from Definition \ref{defn:energy_minimizing_voltage}, it is sufficient to show that the following two statements hold.
\begin{enumerate}
\item If $v^* = T_* v^*$ then $E(v^*) < E(u)$ for all $u\in V_{M_s, M_g}$, $u\neq v^*$
\item If $E(v^* + u) \geq E(v^*)$ for all $u\in V_{M_s, M_g}$ then $v^* = T_* v^*$ 
\end{enumerate}

Let $\varepsilon > 0$ and $v\in V_{M_s, M_g}$. Consider $u=v^* + \varepsilon v\in V_{M_s, M_g}$ and the energy from Definition \ref{defn:energy_minimizing_voltage} evalauted at $u$ namely
\begin{align*}
\begin{split}
        E(u) = E(v^* + \varepsilon v) &= \int_{M}\int_{M}\nk(x,y)[(v^*(x) + \varepsilon v(x))- (v^*(y) + \varepsilon v(y))]^2d\mu(x)d\mu(y) \\
        & = \int_{M}\int_{M}\nk(x,y)(v^*(x)-v^*(y))^2d\mu(x)d\mu(y) \\
        & - 2\varepsilon \int_{M}\int_{M}\nk(x,y)(v^*(x)-v^*(y))(v(x)-v(y))d\mu(x)d\mu(y) \\
        & + \varepsilon^2\int_{M}\int_{M}\nk(x,y)(v(x)-v(y))^2d\mu(x)d\mu(y) \\
        & = E(v^*) + \varepsilon^2 E(v) - 4\varepsilon E(v^*, v).
\end{split}
\end{align*}
Here $E(v^*, v) \coloneqq \frac{1}{2}\int_{M}\int_{M}\nk(x,y)(v^*(x)-v^*(y))(v(x)-v(y))d\mu(x)d\mu(y)$. We then have

\begin{equation}\label{eq:energy_decomposition}
    E(v^*) = E(v^* + \varepsilon v) - \varepsilon^2 E(v) + 4\varepsilon E(v^*, v)
\end{equation}

For the first statement, if $v=T_* v^*$ it follows from Lemma \ref{lemma:proof_bilinar_minimizer_is_vstar} that $E(v^*, v) = 0$. Consequently, with $\varepsilon=1$,
\begin{equation*}
    E(v^*) = E(v^* + \varepsilon v) - E(v).
\end{equation*}
Since $E(v^* + \varepsilon v)>0$ and $E(v)>0$ it follows that $E(v^*) < E(v^* + \varepsilon v)$. Because $u$ was arbitrary, the first statement holds. 

For the second statement, if $E(v^* + \varepsilon v) \geq E(v^*)$ then from Eq. \eqref{eq:energy_decomposition} we require $4\varepsilon E(v^*, v) - \varepsilon^2 E(v) \leq 0$. In other words, $\varepsilon E(v) \geq 4 E(v^*, v)$ for any $ \varepsilon > 0$. Since $\varepsilon$ can be arbitrarily small, this means we need $E(v^*, v)=0$. From Lemma \ref{lemma:proof_bilinar_minimizer_is_vstar} it follows that if $E(v^*, v)=0$ for all $v\in V_{M_s, M_g}$ then $v^* = T v^*$. Consequently, the second statement holds.
\end{proof}

\begin{lemma}\label{lemma:proof_bilinar_minimizer_is_vstar}
    Consider the bilinear energy $E(v^*, v) = \frac{1}{2}\int_{M}\int_{M}\nk(x,y)(v^*(x)-v^*(y))(v(x)-v(y))d\mu(x)d\mu(y)$ defined in lemma \ref{lemma:proof_that_vstar_is_minimizer}.
    \begin{enumerate}
        \item If $v^* = T_* v^*$ then $E(v^*, v) = 0$ for all $v\in V_{M_s, M_g}$
        \item If $E(v^*, v) = 0$ for all $v\in V_{M_s, M_g}$ then $v^* = T_* v^*$
    \end{enumerate}
\end{lemma}
\begin{proof}
    From Lemma \ref{lemma:expression_for_bilinear_energy_form} it follows that
    \begin{equation*}
        E(v^*, v) = \int_{M} v(x)[v^*(x) - (Tv^*)(x)]d\mu(x).
    \end{equation*}
    Consequently, if $v^* = T_* v^*$, then $E(v^*, v) = \int_{M} v(x)[v^*(x) - v^*(x)]d\mu(x) = 0$ for all $v\in V_{M_s, M_g}$. Now, consider the second statement. If $E(v^*, v) = 0$ for all $v\in V_{M_s, M_g}$ then we require $v^* - Tv^* = 0$. Consequently, $v^* = Tv^*$.
\end{proof}

\begin{lemma}\label{lemma:expression_for_bilinear_energy_form}
    The bilinear energy form
    \begin{equation*}
        E(v^*, v) = \frac{1}{2}\int_{M}\int_{M}\nk(x,y)(v^*(x)-v^*(y))(v(x)-v(y))d\mu(x)d\mu(y)
    \end{equation*}
    can be written as 
    \begin{equation*}
        E(v^*, v) = \int_{M} v(x)[v^*(x) - (Tv^*)(x)]d\mu(x).
    \end{equation*}
\end{lemma}
\begin{proof}
We have
\begin{align*}
    \begin{split}
        E(v^*, v) &= \frac{1}{2} \int_{M}\int_{M}\nk(x,y) v^*(x) v(x) d\mu(x)d\mu(y) +  \frac{1}{2} \int_{M}\int_{M}\nk(x,y) v^*(y) v(y) d\mu(x)d\mu(y) \\
        & -\frac{1}{2} \int_{M}\int_{M}\nk(x,y) v^*(x) v(y) d\mu(x)d\mu(y) - \frac{1}{2} \int_{M}\int_{M}\nk(x,y) v^*(y) v(x) d\mu(x)d\mu(y) \\
        & = \underbrace{\int_{M}\int_{M}\nk(x,y) v^*(x) v(x) d\mu(x)d\mu(y)}_{I_1} - \underbrace{\int_{M}\int_{M}\nk(x,y) v^*(y) v(x) d\mu(x)d\mu(y)}_{I_2}
    \end{split}
\end{align*}

For $I_1$ we have
\begin{align*}
    \begin{split}
    \int_{M}\int_{M}\nk(x,y) v^*(x) v(x) d\mu(x)d\mu(y) & =  \int_{M} \bigg(\int_{M}\nk(x,y)d\mu(y)\bigg) v^*(x) v(x) d\mu(x) \\
    & = \int_{M}v^*(x) v(x) d\mu(x).
    \end{split}
\end{align*}
In the last step, we used that $\int_{M}\nk(x,y)d\mu(y) = 1$, since from Definition \ref{defn:norm_kernel} we have $\nk(x, y) = k(x,y)/\int_\M k(x, z)d\mu(z)$. For $I_2$ we have
\begin{align*}
    \begin{split}
    \int_{M}\int_{M}\nk(x,y) v^*(y) v(x) d\mu(x)d\mu(y) 
        & = \int_{M} v(x)\bigg(\int_{M} \nk(x,y) v^*(y) d\mu(y)\bigg) d\mu(x)\\
        & = \int_{M} v(x) (T_* v^*)(x) d\mu(x)
    \end{split}
\end{align*}

It follows that 
\begin{equation*}
    E(v^*, v) =\int_{M} [v^*(x) v(x) - v(x) (Tv^*)(x)]d\mu(x) = \int_{M} v(x)[v^*(x) - (Tv^*)(x)]d\mu(x)
\end{equation*}
\end{proof}

\section{Proofs from Analysis Section}
In this section, our goal will be to show that effective resistances computed over metric graphs that are sampled from $M$ will converge towards the desired limit object. We begin with some technical results for approximating integrals with sums over finite samples from $\mu$. 

\subsection{Bounding integrals of functions}

We begin by stating Hoeffding's bound for a map $f: \M \to [0, 1]$, $\int_\M f(y) d\mu(y)$ can be approximated with the sum $\frac{1}{n}\sum_{i = 1}^n f(x_i).$ 

\begin{lemma}[Hoeffding]\label{lem:hoeffding}
Let $S_n \sim \mu^n$ be $n$ i.i.d draws with $S_n = \{x_1, \dots, x_n\}$, and let $f: \M \to [0, 1]$ be a map. Then,
$$\Pr\Bigg[\left|\int_\M f(y) d\mu(y) -  \frac{1}{n}\sum_{i = 1}^n f(x_i)\right| > \epsilon\Bigg] \leq 2\exp\left( -2n\epsilon^2\right).$$
\end{lemma}

Next, we bound the quotient of two integrals.
\begin{lemma}\label{lem:quotient_hoeffding}
Let $S_n \sim \mu^n$ be $n$ i.i.d draws with $S_n = \{x_1, \dots, x_n\}$, and let $f: \M \to [0, 1]$ and $g: \M \to [0, 1]$ be two maps. Suppose that $\int_\M g(y)d\mu(y) \geq G$.  Then $$\Pr\Bigg[\left|\frac{\int_\M f(y) d\mu(y)}{ \int_\M g(y) d\mu(y)} -  \frac{\frac{1}{n}\sum_{i = 1}^n f(x_i)}{\frac{1}{n}\sum_{i = 1}^n g(x_i)}\right| > \epsilon\Bigg] \leq 4\exp\left( -8nG^2\epsilon^2\right).$$
\end{lemma}

\begin{proof}
It suffices to bound the probability that $$\left|\int_\M f(y) d\mu(y) -  \frac{1}{n}\sum_{i = 1}^n f(x_i)\right| < \frac{\epsilon}{4}$$ and $$\left|\int_\M g(y) d\mu(y) -  \frac{1}{n}\sum_{i = 1}^n g(x_i)\right| < \frac{G\epsilon}{4}$$ as these two equations imply the desired bound by straightforward algebra. The probability that both of these occur can be bounded by a union bound after applying Hoeffding's inequality twice. 
\end{proof}

\subsection{Showing that $A_n^m$ is a contraction}

We define the \textit{finite} analogs of $A_*$ and $b_*$ (Definition \ref{defn:affine}). 

\begin{definition}\label{defn:affine_finite}
Let $S_n \sim \mu^n$ be a sample of $n$ points, $x_1, \dots, x_n$. Let $v: S_n \to \reals$ be a map. Then the operators $A_n$ and the map $b_n$ are defined as
$$(A_nv)(x) = \frac{\sum_{i=1}^n k(x, x_i)v(x_i)\ind\left(x_i \in \M \setminus (\M_0 \cup \M_1)\right)}{\sum_{i=1}^n k(x, x_i)},$$ and $$b_n(x) = \frac{\sum_{i=1}^n k(x, x_i)\ind(x_i \in \M_1)}{\sum_{i=1}^n k(x, x_i)}.$$
\end{definition}

Our goal in this section is to show that $A_n^m$ is also a contraction, where $m$ is as defined in the statement of Theorem \ref{thm:convergence}.

We first define the parameter $K$ as follows.
\begin{definition}
Let $K$ be the inverse of the minimal degree in $\M$. That is, let $$K = \max_{x \in \M} \frac{1}{\int_\M k(x,y)d\mu(y)}.$$
\end{definition}

This is well defined because $\M$ is compact and because $\int_\M k(x, y)d\mu(y)$ is both continuous and larger than $0$ by assumption. Because $k$ has range in $[0, 1]$, it follows that \begin{equation}\label{eqn:max_k}\max_{x \in \M} \nk(x) \leq K\end{equation} where $\nk$ is the normalized kernel defined in Definition \ref{defn:norm_kernel}.

We now prove a technical lemma that will assist us in understanding the structure of $G_n$, the metric graph built over a finite sample $S_n \sim \mu^n$. 

\begin{lemma}\label{lem:back_step}
Let $P \subseteq \M$ be a measurable set, and let $i > 1$. Let $p > 0$ and suppose that for all $x \in \M$, $$\int_\M \nk^{(i)}(x, y)\ind(y \in P)d\mu(y) \geq p.$$ Then there exists a measurable subset $Q \subseteq \M$ such that the following two properties hold: 
\begin{itemize}
	\item For all $x \in \M$, $\int_\M \nk^{(i-1)}(x, y)\ind(y \in Q)d\mu(y) \geq \frac{p}{2K - p}$, with $K$ as defined above.
	\item For all $x \in Q$, $\int \nk(x, y)\ind(y \in P)d\mu(y) \geq \frac{p}{2}$. 
\end{itemize}
\end{lemma}

\begin{proof}
Observe that by the definition of convolution and Fubini's theorem,
\begin{equation*}
\begin{split}
\int_\M \nk^{(i)}(x, y)\ind(y \in P)d\mu(y) &= \int_M (\nk^{(i-1)} \circ \nk)(x,y)\ind(y \in P)d\mu(y) \\
&= \int_\M \left(\int_\M\nk^{(i-1)}(x, z)\nk(z, y)d\mu(z)\right) \ind(y \in P)d\mu(y) \\
&= \int_\M \nk^{(i-1)}(x, z) \left(\int_\M \nk(z, y)\ind(y \in P)d\mu(y)\right)d\mu(z).
\end{split}
\end{equation*}
Define $f(z) = \int_\M \nk(z, y)\ind(y \in P)d\mu(y)$. Since $\nk \leq K$, it follows that $f$ has range $[0, K]$. We now define $Q = \{z: f(z) \geq \frac{p}{2}\}$, and claim that this suffices.
By definition, the second property trivially holds. To see that the first property holds, let $q = \int_\M \nk^{(i-1)}(x, y)\ind(y \in Q)d\mu(y)$. Then,
\begin{equation*}
\begin{split}
p &\leq \int_\M \nk^{(i)}(x, y)\ind(y \in P)d\mu(y) \\
&= \int_\M \nk^{(i-1)}(x, z) \left(\int_\M \nk(z, y)\ind(y \in P)d\mu(y)\right)d\mu(z) \\
&= \int_\M \nk^{(i-1)}(x, z) f(z)d\mu(z) \\
&=  \int_Q \nk^{(i-1)}(x, z) f(z)d\mu(z) +  \int_{\M \setminus Q} \nk^{(i-1)}(x, z) f(z)d\mu(z) \\
&\leq Kq + \left(\frac{p}{2}\right)(1 -q).
\end{split}
\end{equation*}
Rearranging this gives the desired inequality.
\end{proof}

Recall that $m$ and $\alpha$ are defined in the statement of Theorem \ref{thm:convergence}  such that $\M = C_\alpha^m(\M_1 \cup \M_0)$. We have the following lemma which gives additional structure to $\M$.

\begin{lemma}\label{lem:layout}
There exists sets $P_0, P_1, P_2, \dots, P_m \subseteq \M$ with $P_0 = \M$ and $P_m = \M_0 \cup \M_1$ such that the following holds. There exists positive reals $p_1, p_2, \dots p_m$ such that for all $1 \leq i \leq m$, for all $x \in P_{i-1}$, $$\int_\M \nk(x, y)\ind(y \in P_{i})d\mu(y) \geq p_i.$$ 
\end{lemma}

\begin{proof}
The idea is simple: we apply Lemma \ref{lem:back_step} $m$ times starting with $P_m = \M_0 \cup \M_1$. To establish the base case, recall that $\M = C_\alpha^m(\M_1 \cup \M_0)$, which implies $$\int_M \nk^{(m)}(x, y)\ind(y \in P_m)d\mu(y) \geq \alpha$$ for $\alpha > 0$. 

Define $q_m = \alpha$. We will now show how to recursively construct $P_i$, $q_{i-1}$ and $p_i$ from $q_i$. Suppose that for all $x \in \M$, \begin{equation}\label{eqn:recurse}\int_\M \nk^{(i)}(x,y)\ind(y \in P_i)d\mu(y) \geq q_i.\end{equation} Then applying Lemma \ref{lem:back_step}, we let $P_{i-1} = Q$, $p_i = \frac{q_{i}}{2}$, and $q_{i-1} = \frac{q_i}{2k - q_i}.$ It is easy to see by Lemma \ref{lem:back_step} that doing so preserves Equation \ref{eqn:recurse} and that $p_i$ satisfies the desired equation. 

Finally, in the case that $i = 1$, we can simply let $p_1 = q_1$, as we no longer require Lemma \ref{lem:back_step} since $\nk^{(1)} = \nk$. This completes the proof.
\end{proof}

We are now prepared to show that $A_n^m$ is indeed a contraction over $S_n$.

\begin{lemma}\label{lem:finite_contraction}
There exists an absolute constant $p > 0$ independent of $n$ and $S_n$ such that the following holds. $$\lim_{n \to \infty} \Pr_{S_n \sim \mu^n}[||A_n^m||_\infty \leq 1- p] = 1.$$
\end{lemma}

\begin{proof}
Fix $P_0, \dots P_m$ and $p_0, p_1, \dots, p_m$ as in Lemma \ref{lem:layout}. Let $E$ denote the event that for all $1 \leq i \leq n$ and $1 \leq j \leq m$ such that  $x_i \in P_{j-1}$, $$\frac{\frac{1}{n}\sum_{t = 1}^n k(x_i, x_t)\ind(x_t \in P_j)}{\frac{1}{n}\sum_{t = 1}^n k(x_i, x_t)} \geq \frac{p_i}{2}.$$ The key observation is that by applying Lemma \ref{lem:quotient_hoeffding} to all $mn$ pairs of points along with a union bound, there exists an absolute constant $C$ for which $\Pr[E] \geq 1 - nm \exp( -nC).$ Thus for $n$ sufficiently large, the probability of $E$ converges to $1$. This is a direct result of the fact that by Lemma \ref{lem:layout}, $$\int_\M \nk(x_i, y)\ind(y \in P_j)d\mu(y) \geq p_j,$$ whenever $x_i \in P_{j-1}$. Note that although there is a technical independence issue as the function $x_t \to k(x_i, x_t)$ has a dependence on $x_i$ which is in $S_n$, this can be resolved by observing that for any function $f$, for $n$ sufficiently large, $\left|\frac{1}{n}\sum f(x_j) - \frac{1}{n-1}\sum_{j \neq i} f(x_j)\right|$ is small. 

Finally, given that the event $E$ occurs, we can bound $||A_n^m||_\infty$ as follows. For any $v: S_n \to \reals$, we have from Definition \ref{defn:affine_finite},
\begin{equation*}
\begin{split}
|A_n^v(x) | &\leq \sum_{t_1 = 1}^n \frac{k(x, x_{t_1})}{\sum_{i=1}^n k(x, x_i)}\sum_{t_2 = 1}^n \frac{k(x_{t_1}, x_{t_2})}{\sum_{i=1}^n k(x_{t_2}, x_i)} \dots \sum_{t_m = 1}^n \frac{k(x_{t_{m-1}}, x_{t_m})}{\sum_{i=1}^n k(x_{t_m}, x_i)} \ind(x_{t_m} \in \M \setminus (\M_1 \cup \M_0) \\
&= 1 - \sum_{t_1 = 1}^n \frac{k(x, x_{t_1})}{\sum_{i=1}^n k(x, x_i)}\sum_{t_2 = 1}^n \frac{k(x_{t_1}, x_{t_2})}{\sum_{i=1}^n k(x_{t_2}, x_i)} \dots \sum_{t_m = 1}^n \frac{k(x_{t_{m-1}}, x_{t_m})}{\sum_{i=1}^n k(x_{t_m}, x_i)} \ind(x_{t_m} \in \M_1 \cup \M_0) \\
&\leq 1 - \sum_{x_{t_1} \in P_1} \frac{k(x, x_{t_1})}{\sum_{i=1}^n k(x, x_i)}\sum_{x_{t_2} \in P_2} \frac{k(x_{t_1}, x_{t_2})}{\sum_{i=1}^n k(x_{t_2}, x_i)} \dots \sum_{x_{t_m} \in P_m} \frac{k(x_{t_{m-1}}, x_{t_m})}{\sum_{i=1}^n k(x_{t_m}, x_i)} \ind(x_{t_m} \in \M_1 \cup \M_0) \\
&\leq 1 - \frac{\prod_{i=1}^m p_i}{2^m}.
\end{split}
\end{equation*}
Thus selecting $p = \frac{\prod_{i=1}^m p_i}{2^m}$ suffices, as desired. 
\end{proof}

As a quick remark, although the factor of contraction shown here is \textit{extremely} close to 1, this analysis is just considering worst-case situations in which $\M$ has a very complicated structure. In most actual cases, the factor of contraction is far far smaller.

\subsection{Finishing the proof}

We are now prepared to compare $v^*$ with $v^*_n$. They key idea is to restrict $v^*$ to $S_n$ and then compare it with $A_n^m v^*$. We begin with the following lemma.

\begin{lemma}\label{lem:converge_single_point}
Let $u: M \to [0, 1]$ be a measurable function, $x \in M$ be a point, and $\epsilon > 0$ be a real number. Then there exists a constant $C$ such that $$\Pr_{S \sim \mu^n}\left[\big |(A_*u + b_*)(x) - (A_nu + b_n)(x)\big | > \epsilon \right] < 4 \exp(-Cn\epsilon^2).$$ 
\end{lemma}

\begin{proof}
We will examine $|A_*u - A_nu|$ and $|b_* - b_n|$ separately and then apply the triangle inequality. Let $\M' = \M \setminus (\M_1 \cup \M_0).$ Then,  
\begin{equation*}
\begin{split}
\big |(A_*u)(x) - (A_nu)(x)\big | &= \left|\frac{\int_\M' k(x, y)u(y)d\mu(y)}{\int_\M k(x,y)d\mu(y)} - \frac{\frac{1}{n}\sum_{x_i \in \M'} k(x, x_i)u(x_i)}{\frac{1}{n}\sum_{x_i \in \M'} k(x,x_i)}\right|.
\end{split}
\end{equation*}
However, by Lemma \ref{lem:quotient_hoeffding}, this quantity is at most $\epsilon$ with probability $2\exp(-O(n\epsilon^2))$. We can apply a similar argument for $b_* - b_n$ which completes the proof.
\end{proof}

\begin{lemma}\label{lemma:the_stuff_converges_yay}
Restrict $v^*$ and $b_*$ (definition \ref{defn:affine}) to $S_n$. Then $||v^* - (A_nv^* + b_n)||_\infty \to 0$ and $||b_* - b_n||_\infty \to 0$ in probability. Here again, the infinity norms are taken over $S_n$, not $\M$. 
\end{lemma}

\begin{proof}
Simply apply Lemma \ref{lem:converge_single_point} for all $x = x_i$ and $u = v^*$ and then apply a union bound. While there are technically independence issues, this is solved by observing that $x_i$ is independent of all other elements of $S_n$, and the averages taken over $S_n$ \textit{ignoring} $x_i$ are barely different from those including $x_i$. 
\end{proof}

\begin{lemma}
Let $c_n = b_n + A_nb_n + A_n^2b_n + \dots A_n^{m-1}b_n$. Then $||v^* - (A_n^m v^* + c_n)||_\infty \to 0$ in probability. 
\end{lemma}

\begin{proof}
The key idea is that $A_n$ is at most an averaging operator, and consequently $||A_n||\leq 1$. Intuitively, applying $A_n$ to a map $v: S_n \to \reals$ cannot increase $\max_{x \in S_n} |v(x)|$. Let $T_nv = A_nv + b_n$. This implies that for all $u$ and $v$, $$||u - v||_\infty \geq ||T_nu - T_nv||.$$ Applying this $m$ times along with the triangle inequality, we see that $$|v^* - T_n^m v^*| \leq (m-1)|v^* - T_n v^*|.$$  However, since the latter goes to $0$ in probability, it follows that the former does as well. All the remains to see is that $T_n^m v$ precisely equals $A_n^m v^*  c_n$, as desired. 
\end{proof}

We can finally prove Theorem \ref{thm:convergence}.

\begin{proof}
By the previous Lemma, for any $\epsilon > 0$ and $\delta > 0$ there exists $n$ such that with probability $1 - \delta$, $||v^* - (A_n^mv^* + c_n)||_\infty < \epsilon.$ Since $A_n^m$ is a contraction, it follows by Lemma \ref{lem:close_to_fixed_point} that $v^*$ has distance at most $\frac{\epsilon}{p}$ from $v_n^*$, the fixed point of $v \mapsto A_n^mv + c_n$. Since $p$ is fixed, it follows that $|v_n^* - v_n|_\infty \to 0$. Finally, for an arbitrary point $x \in M$, simply applying Lemma \ref{lem:quotient_hoeffding} one last time over the definition of $v_n^*(x)$ implies the desired result.
\end{proof}

\section{Proof of Convergence using $\alpha$-covers}

Here, our goal is to prove Theorem \ref{thm:alpha_cover_converge}. To do so, we impose an additional technical restriction on our kernel $k$.

\begin{definition}
Let $\mu$ be a measure over metric space $M$, and let $k: M \times M \to [0, 1]$ be a kernel. Let $V \subset M \times M$ denote the set of all $(x, y)$ for which $k$ is continuous at $(x,y)$, and let $\mu \times \mu$ denote the product measure over $M \times M$. Then $k$ is said to be \textbf{$\mu$-continuous} if $$(\mu \times \mu)\left(\left(M \times M\right) \setminus V\right) = 0.$$ 
\end{definition}

It is easy to verify that for most common measures over $\R^d$ that the Gaussian and Radial kernels are $\mu$-continuous. 

We now assume that $k$ is a $\mu$-continuous kernel, and moreover that 

Next, we define an analog to the quantities $A_n$ and $b_n$ with respect to our $\alpha$-cover $\calC$. 

\begin{definition}\label{def:An_on_cover_and__n_on_cover}
Let $S \sim \mu^n$ be an i.i.d sample of data, and let $\calC = \{c_1, \dots, c_m\}$ be an $\alpha$-cover. Let $\gamma_{i,n}$ denote the fraction of points from $S$ that lie in the cell corresponding to $c_i$. Let $v: \calC \to \reals$ be a map. Then the operators $A_{\calC, n}$ and the map $b_{\calC, n}$ are defined as
$$(A_{\calC, n} v)(x) = \frac{\sum_{i=1}^m k(x, c_i)v(c_i)\gamma_{i,n}\ind\left(c_i \in \M \setminus (\M_0 \cup \M_1)\right)}{\sum_{i=1}^m \gamma_{i,n} k(x, c_i)},$$ and $$b_{\calC, n}(x) = \frac{\sum_{i=1}^m \gamma_{i,n} k(x, c_i)\ind(c_i \in \M_1)}{\sum_{i=1}^m \gamma_{i,n}k(x, c_i)}.$$
\end{definition}

The goal is to show that computing the voltage function $v$ using $\alpha$-covers converges towards the same limit object as computing $v$ using a direct sample. To achieve this, the idea is to show that the operators, $A_{\calC, n}$ and $A_n$ (along with the functions $b_{\calC, n}, b_n$) behave similarly. 

\begin{lemma}\label{lemma:alpha_cov_existence_of_Delta_1}
Let $\calC$ be any $\alpha$-cover. There exists a function $\Delta_1$ such that the following two things hold. First, for any $v: M \to [0, 1]$, and any $x \in M$,  $$|(A_{\calC, n} v)(x) - (A_n v)(x)| < \Delta_1(\alpha),\quad |b_{\calC, n}(x) - b_n(x)| < \Delta_1(\alpha).$$ Second, $\lim_{\alpha \to 0} \Delta_1(\alpha) = 0$. 
\end{lemma}

\begin{proof}
This directly follows from the fact that $k$ is $\mu$ continuous. Since the diameter of each cell is at most $\alpha$, and since the support of $M$ is compact, we see that the maximum deviation made in a single entry in the corresponding matrices for $A_{\calC, n}$ and $A_n$ is bounded by some continuous function of $\alpha$.
\end{proof}

\begin{lemma}
There exists a function $\Delta: \R^+ \to \R^+$ such that the following hold. First, for any $\epsilon, \delta$ and $\alpha$-cover, $\calC$, there exists N such that for all $n \geq N$, with probability at least $1-\delta$ over $S \sim \mu^n$, $$|v_n(x) - v_n^\calC(x)| \leq \Delta(\alpha) + \epsilon.$$ Second, $\lim_{\alpha \to 0} \Delta(\alpha) = 0$.
\end{lemma}

\begin{proof}
Let $m, p, N$ be such that for all $n \geq N$, $||A_n^m||_\infty \leq 1 - p$ with probability at least $1-\delta$ over $S \sim \mu^n$. Such $N$ must exist by Lemma \ref{lem:finite_contraction}. It follows that with probability at least $1-\delta$, $T_n^m$ is a \textit{nice} operator (Definition \ref{defn:nice}). 

By the definitions of $v_n$ and $v_n^\calC$, we have that $T_n^m v_n = v_n$ and $T_{n, \calC}^m v_n^\calC = v_n^\calC$, where $T_{n, \calC}(v) = A_{n, \calC}(v) + b_{n, \calC}$. It follows by applying the previous lemma $m$ times that $$|T_n^m(v_n^\calC) - v_n^\calC| = |T_n^m(v_n^\calC) - T_{n, \calC}^m(v_n^\calC)| \leq O(m\Delta_1(\alpha)).$$ Since $T_n^m$ is nice with norm at most $1- p$, it follows by Lemma \ref{lem:close_to_fixed_point} that $$|v_n^\calC - v_n| \leq O\left(m\frac{\Delta_1(\alpha)}{p}\right).$$ Since $m$ and $p$ can be chosen to be fixed and since $\lim_{\alpha\rightarrow\infty}\Delta_1(\alpha) \rightarrow 0$ from Lemma \ref{lemma:alpha_cov_existence_of_Delta_1}, it follows that $\Delta(\alpha) \mapsto m\frac{\Delta_1(\alpha)}{p}$ satisfies both properties above which completes the proof.
\end{proof}

\begin{lemma}
The sequence $v_i^\calC(x)$ converges in probability. 
\end{lemma}

\begin{proof}
Let the operator $A_{\calC}$ and the map $b_{\calC}$ be defined as
$$(A_{\calC} v)(x) = \frac{\sum_{i=1}^m k(x, c_i)v(c_i)\gamma_i\ind\left(c_i \in \M \setminus (\M_0 \cup \M_1)\right)}{\sum_{i=1}^m \gamma_i k(x, c_i)},$$ and $$b_{\calC}(x) = \frac{\sum_{i=1}^m \gamma_i k(x, c_i)\ind(c_i \in \M_1)}{\sum_{i=1}^m \gamma_ik(x, c_i)}.$$ Here, we let $\gamma_i$ denote the probability mass under $\mu$ of the cell corresponding to $c_i$ (Note the difference with $\gamma_{i,n}$ defined in Definition \ref{def:An_on_cover_and__n_on_cover}). Then by applying an proof analogous to that in Lemma \ref{lem:converge_single_point}, we have that $$\Pr_{S \sim \mu^n} \left[|(A_\calC v + b_\calC)(x) - (A_{n, \calC}u + b_{n, \calC}) > \epsilon\right] < 4\exp\left(-Cn\epsilon^2\right).$$ Applying an argument analogous to that in Lemma \ref{lemma:the_stuff_converges_yay} finishes the proof. 
\end{proof}

Finally, Theorem \ref{thm:alpha_cover_converge} is a direct consequence of the previous two lemmas.

\section{Properties of the region-based effective resistance}
\subsection{The reduced graph}\label{section:reduced_graph}
\cite{song2019extension} offer an interpretation of the effective resistance between sets $X_a, X_b$ on a graph $G$ as the effective resistance between two aggregated nodes $a,b$ on a reduced graph $G_{ab}$. We extend the concept of a reduced graph to more than two nodes by considering the reduced graph corresponding to the sets $X_a, X_b, X_z$. 

We define $P_{abz} = \begin{bmatrix} \mathbbm{1}_n(X_a) & \mathbbm{1}_n(X_b) & \mathbbm{1}_n(X_z) & I_c \end{bmatrix} \in \bbR^{n\times (c+3)}$, where $\mathbbm{1}_n(X_p)\in \bbR^n$ is the indicator vector
\begin{equation}
     \mathbbm{1}_n(X_p) = \begin{cases} 1, & \forall i \in X_p \\
    0, & \text{otherwise}
    \end{cases}
\end{equation}
and $I_c \in \bbR^{n\times c}$ is the matrix collecting all basis vectors $e_i$ for $i \in X_c$. 

\begin{definition}(Reduced graph)\label{def:reduced_graph}
We define the reduced graph of $G$, with respect to the non-empty disjoint subsets $X_a, X_b, X_z \subset X$, to be the graph $G_{abz}$ with Laplacian $L_{G_{abz}} \coloneqq P_{abz}^\top L_G P_{abz}$.
\end{definition}

The relation between the reduced graph $G_{abz}$ and $G$ is given by Lemma \ref{lemma:interp_reduced_graph}. We note that the interpretation offered by Lemma \ref{lemma:interp_reduced_graph} will be used in our experiments presented in Section \ref{section:experiments}, where the degree of the source and sink sets will be taken as the degree of the corresponding aggregated nodes.

\begin{lemma}\label{lemma:interp_reduced_graph}
The reduced graph $G_{abz}$ corresponds to reducing the nodes in $G$ contained in each of the subsets 
 $X_p$ for $p\in \{a, b, z\}$ into corresponding aggregated nodes $a$, $b$, and $z$ that satisfy the following properties for each aggregated node.
 \begin{itemize}
     \item The degree $D_p$ of each aggregated node $p$ is $D_{p}= \sum_{i\in X_p} D_i - \sum_{i\in X_p}\sum_{j\in X_p, j\neq i} W_{ij}$
     \item All external edges from each of the sets $X_p$ to the complement $X\backslash X_p$ are preserved such that
     \begin{itemize}
        \item The edge weights of each aggregated node $p$ to a node $i \in X_c$ are given by $W_{pi}=\sum_{j\in X_p} W_{ji}$
        \item The edge weights of each aggregated node $p$ to another aggregated node $q\neq p$ is $\sum_{j\in X_p}\sum_{j\in X_q} W_{ji}$
     \end{itemize}
 \end{itemize}
\end{lemma}
\begin{proof}
We write $L_G P_{abz}$ in block form $L_G P_{abz} = \begin{bmatrix}
    L\mathbbm{1}_n(X_a) & L\mathbbm{1}_n(X_b) & L\mathbbm{1}_n(X_z) & L I_c
\end{bmatrix} \in \bbR^{n \times (c+3)}$. This gives 
\begin{equation*}
    L_{G_{abz}} = P_{abz}^\top LP_{abz} = \left[\begin{array} {ccc|c}
    L_{aa} & L_{ab} & L_{az} & L_{ac} \\
    L_{ba} & L_{bb} & L_{bz} & L_{bc} \\
    L_{za} & L_{zb} & L_{zz} & L_{zc} \\ \hline
    L_{ca} & L_{cb} & L_{cz} & L_{cc} \\
\end{array}\right]
\end{equation*}
where $L_{pq} = \sum_{i\in X_p}\sum_{j\in X_q} (L_G)_{ij}$ for $p,q \in \{a,b,z\}$, $L_{cq}(i) = \sum_{j\in X_q} (L_G)_{ij}$ for $i=1,\cdots c$ and $L_{cc}\in \bbR^{c\times c}$ is the Laplacian on $X_c$. Now $(L_G)_{ij} = -W_{ij}$ for $i\neq j$, where $W_{ij}$ is the edge weight between node $i,j$, and $L_{ij} = D_i$ for $i=j$, where $D_i$ is the degree of node $i$. The interpretation of the blocks follows.
\end{proof}

\begin{lemma}\label{lemma:relating_r_G_with_rG_reduced}
The effective resistance between the nodes $p, q \in \{a, b, z\}$ in the graph $G_{abz}$
corresponds to the effective resistance between the corresponding non-empty disjoint sets $X_p, X_q$ in the graph $G$ such that $R_{G_{abz}}(p, q)= R^s_G(X_p, X_q)$.
\end{lemma}
\begin{proof}
From Theorem \ref{thm:def_er_between_sets}, Definition \ref{def:reduced_graph} and Lemma \ref{lemma:interp_reduced_graph} it follows that
\begin{equation*}
     R^s_G(X_a, X_b) = (e_1-e_2)^\top ( P_{abz}^\top L_G P_{abz})^\dagger(e_1-e_2) = (e_1-e_2)^\top (L_{G_{abz}})^\dagger(e_1-e_2) = R_{G_{abz}}(p,q).
\end{equation*}
\end{proof}

\subsection{Proof of Theorem \ref{thm:triangle_ineq_for_ER_sets} Triangle inequality for ER between sets}\label{appendix:proof_triangle_ineq_er_sets}
Consider the setting in Section \ref{section:ER_sets}. The effective resistance between sets can be calculated using the voltage difference formulation of Proposition \ref{prop:poisson_formulation_ER},
by extending the source and sink constraints to the subsets. We require $\sum_{i\in X_s} J_i = 1$ and similarly $\sum_{i\in X_g} J_i = -1$. With $J_i=(Lv)_i$ from Eq. \eqref{eq:combined_kirchhoff_and_ohm} we have $\sum_{i\in X_s} (Lv)_i = 1$ and $\sum_{i\in X_g} (Lv)_i = -1$. Proposition \ref{prop:poisson_formulation_ER_sets} defines the new optimization problem we need to solve.

\begin{proposition}[Voltage difference formulation on sets]\label{prop:poisson_formulation_ER_sets}
The effective resistance between the non-empty disjoint subsets $X_s, X_g \subset X$ corresponds to $R^s(X_s, X_g) = E(\nu)$, where $\nu$ is the voltage that minimizes the energy
\begin{align*}
    \begin{split}
       \min_{v} & \quad \sum_{x_i, x_j \in X} \W_{i,j} (\v(x_i)-\v(x_j))^2 \\
        \text{Subject to} & \quad \sum_{i\in X_s} (Lv)_i = 1, \quad  \sum_{i\in X_g} (Lv)_i = -1, \quad (L v)_i = 0, \forall i\in X_c
    \end{split}
\end{align*}
\end{proposition}
\begin{proof}
Theorem 2 in \cite{song2019extension}.
\end{proof}

The next theorem extend the result in Theorem 2, Song et al.~\cite{song2019extension} to an additional subset $X_z \subset X$. Where $X_z$ is non-empty and disjoint from $X_a, X_b$. 

\begin{theorem}\label{thm:def_er_between_sets}
The effective resistance from Proposition 
 \ref{prop:poisson_formulation_ER_sets} can be written as
\begin{equation*}
        R^s_G(X_a, X_b) = (e_1-e_2)^\top ( P_{abz}^\top L_G P_{abz})^\dagger(e_1-e_2)
\end{equation*}
where $e_i \in \bbR^{c+3}$ is a basis vector with $1$ at $i$ and zero otherwise.
\end{theorem}
\begin{proof}
Using $P_{abz}$, we can write the constraints in Proposition \ref{prop:poisson_formulation_ER_sets} in a more compact form $P_{abz}^\top L = (e_1-e_2)$. Applying these constraints with the Lagrange multiplier $\gamma \in \bbR^{c+3}$ we have
\begin{equation*}
    f(v,\gamma) = v^\top L v + \gamma (P_{abz}^\top L - e_1+e_2)
\end{equation*}
Since the optimization problem is convex, the effective resistance can be found by solving
\begin{equation*}
    R^s_G(X_a, X_b) = \min_v \max_\gamma f(v,\gamma)
\end{equation*}
The remaining steps can be found in the proof of Theorem 2, Song et al.~\cite{song2019extension}.
\end{proof}

Having established the setting, we can move on to the main proof.

\begin{proof}
Consider the reduced graph $G_{abz}$ of $G$ from Definition \ref{def:reduced_graph}. In this reduced graph the current is injected into the nodes aggregated in $a$ and is extracted from the nodes aggregated in $b$. We denote the total current $y_{ab}$. From the conservation of current, we must have that the current $y_{ak}$ and $y_{bk}$ through any internal node $k\in X_c \cup \{z\}$ must satisfy $y_{ab} \geq y_{ak}$ and $y_{ab} \geq y_{bk}$.  Let $v(i)$ be the voltage at node $i$. For the current to flow from $a$ to $b$, we need the voltage at the nodes to satisfy $v(a) > v(k) > v(b)$ for all $i\in X_c \cup \{z\}$.

The effective resistance of the graph $G_{abz}$ is an intrinsic property of the graph and does not depend on the choice of source and sink. We let $a$ be the source and $b$ the sink. We then have
\begin{equation*}
    y_{ab} = \frac{v(a)-v(b)}{R_{G_{abz}}(a,b)} \geq \frac{v(a)-v(z)}{R_{G_{abz}}(a,b)} \geq y_{az} \quad \text{and} \quad y_{ab} = \frac{v(a)-v(b)}{R_{G_{abz}}(a,b)} \geq \frac{v(z)-v(b)}{R_{G_{abz}}(a,b)} \geq y_{bz}
\end{equation*}
This gives
\begin{equation*}
    \frac{v(a)-v(z)}{v(a)-v(b)} \leq \frac{R_{G_{abz}}(a, z)}{R_{G_{abz}}(a, b)} \quad \text{and} \quad     \frac{v(z)-v(b)}{v(a)-v(b)} \leq \frac{R_{G_{abz}}(b, z)}{R_{G_{abz}}(a, b)}
\end{equation*}
Adding the two inequalities gives
\begin{equation*}
    R_{G_{abz}}(a, b) \leq R_{G_{abz}}(a, z) + R_{G_{abz}}(z, b)
\end{equation*}
Using Lemma \ref{lemma:relating_r_G_with_rG_reduced} we have that $R_{G_{abz}}(p, q) = R^s_G(X_p, X_q)$ for $p,q \in \{a,b,z\}$. The result follows.
\end{proof}


\end{document}